\documentclass[10pt,twocolumn]{article}
\pdfoutput=1
\usepackage{cvpr}
\usepackage{times}
\usepackage{epsfig}
\usepackage{graphicx}
\usepackage{amsmath}
\usepackage{amssymb}
\usepackage{amsthm}
\usepackage{bm}
\usepackage{cite}
\usepackage{enumitem}
\usepackage[export]{adjustbox}
\usepackage{tabularx}

\theoremstyle{remark}
\newtheorem{definition}{Definition}

\newtheorem{theorem}{Theorem}[section]

\usepackage[pagebackref=true,breaklinks=true,letterpaper=true,colorlinks,bookmarks=false]{hyperref}

\cvprfinalcopy % *** Uncomment this line for the final submission
\def\cvprPaperID{1129} % *** Enter the CVPR Paper ID here

\ifcvprfinal\fi
\begin{document}

%%%%%%%%% TITLE
\title{FoldingNet: Point Cloud Auto-encoder via Deep Grid Deformation}

\author{
	{Yaoqing Yang\footnotemark[2]{}}\\
		{\tt\small yyaoqing@andrew.cmu.edu}
	\and
	{Chen Feng\footnotemark[3]{}}\\
		{\tt\small cfeng@merl.com}
	\and
	{Yiru Shen\footnotemark[4]{}}\\
		{\tt\small yirus@g.clemson.edu}
	\and
	{Dong Tian\footnotemark[3]{}}\\
		{\tt\small tian@merl.com}
	\and
	\normalsize{\textsuperscript{$\dagger$}Carnegie Mellon University \quad \textsuperscript{$\ddagger$}Mitsubishi Electric Research Laboratories (MERL) \quad \textsuperscript{$\mathsection$}Clemson University}
}

\maketitle
\thispagestyle{empty}

%%%%%%%%% ABSTRACT
\begin{abstract}
	\vspace{-2mm}
   Recent deep networks that directly handle points in a point set, e.g., PointNet, have been state-of-the-art for supervised learning tasks on point clouds such as classification and segmentation. In this work, a novel end-to-end deep auto-encoder is proposed to address unsupervised learning challenges on point clouds. On the encoder side, a graph-based enhancement is enforced to promote local structures on top of PointNet. Then, a novel \emph{folding}-based decoder deforms a canonical 2D grid onto the underlying 3D object surface of a point cloud, achieving low reconstruction errors even for objects with delicate structures. The proposed decoder only uses about 7\% parameters of a decoder with fully-connected neural networks, yet leads to a more discriminative representation that achieves higher linear SVM classification accuracy than the benchmark. In addition, the proposed decoder structure is shown, in theory, to be a generic architecture that is able to reconstruct an arbitrary point cloud from a 2D grid. Our code is available at \url{http://www.merl.com/research/license#FoldingNet}
\end{abstract}
\vspace{-4mm}
%%%%%%%%% BODY TEXT

\section{Introduction}
\vspace{-2mm}
3D point cloud processing and understanding are usually deemed more challenging than 2D images mainly due to a fact that point cloud samples live on an irregular structure while 2D image samples (pixels) rely on a 2D grid in the image plane with a regular spacing. Point cloud geometry is typically represented by a set of sparse 3D points. Such a data format makes it difficult to apply traditional deep learning framework. E.g. for each sample, traditional convolutional neural network (CNN) requires its neighboring samples to appear at some fixed spatial orientations and distances so as to facilitate the convolution. Unfortunately, point cloud samples typically do not follow such constraints. One way to alleviate the problem is to voxelize a point cloud to mimic the image representation and then to operate on voxels. The downside is that voxelization has to either sacrifice the representation accuracy or incurs huge redundancies, that may pose an unnecessary cost in the subsequent processing, either at a compromised performance or an rapidly increased processing complexity. Related prior-arts will be reviewed in Section~\ref{sec:related_works}.
\begin{table}
\begin{tabular}{cccc}
\hline
Input&2D grid&1st folding&2nd folding\\
\hline
\includegraphics[width=.2\columnwidth,height=.2\columnwidth,keepaspectratio,valign=m,margin=.0cm .05cm]{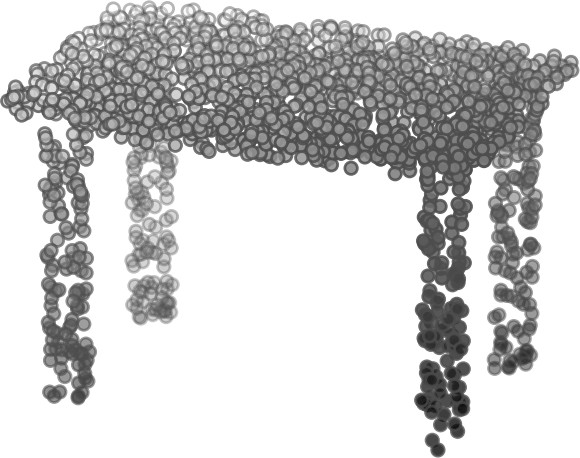}&
&
\includegraphics[width=.2\columnwidth,height=.18\columnwidth,keepaspectratio,valign=m,margin=.0cm .05cm]{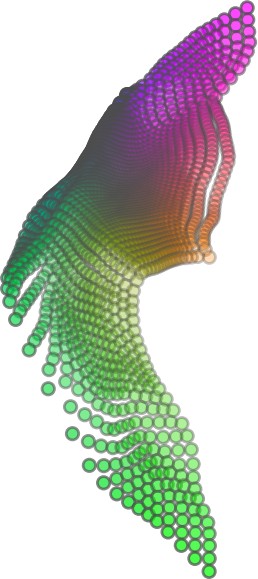}&
\includegraphics[width=.2\columnwidth,height=.2\columnwidth,keepaspectratio,valign=m,margin=.0cm .05cm]{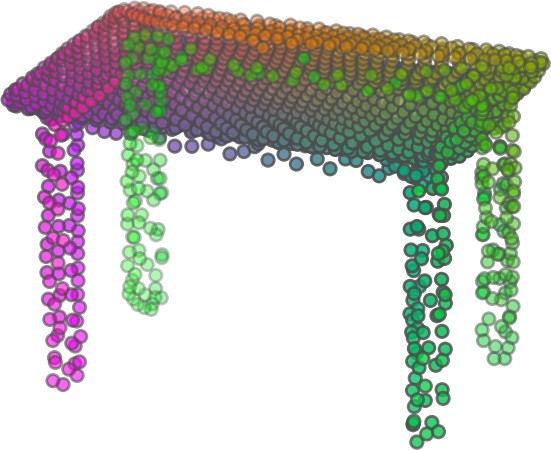}\\
\includegraphics[width=.2\columnwidth,height=.2\columnwidth,keepaspectratio,valign=m,margin=.0cm .05cm]{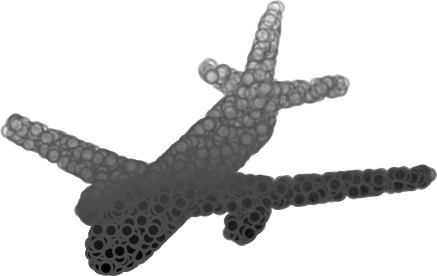}&
&
\includegraphics[width=.2\columnwidth,height=.15\columnwidth,keepaspectratio,valign=m,margin=.0cm .05cm]{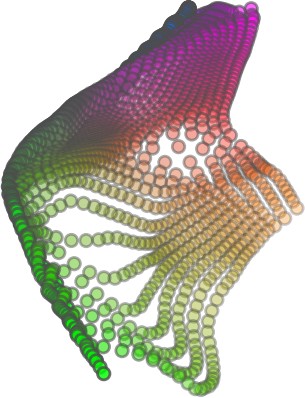}&
\includegraphics[width=.2\columnwidth,height=.2\columnwidth,keepaspectratio,valign=m,margin=.0cm .05cm]{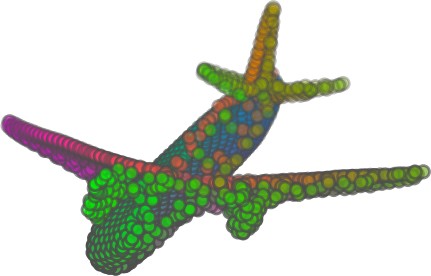}\\
\includegraphics[width=.2\columnwidth,height=.2\columnwidth,keepaspectratio,valign=m,margin=.0cm .05cm]{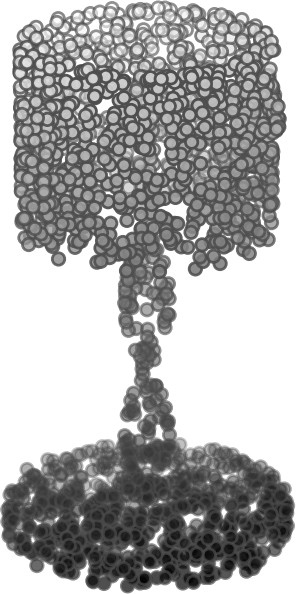}&
\includegraphics[width=.2\columnwidth,height=.2\columnwidth,keepaspectratio,valign=m,margin=.0cm .05cm]{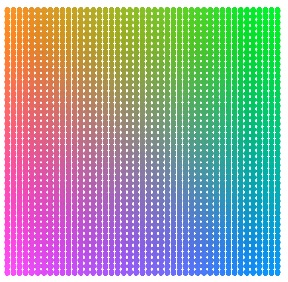}&
\includegraphics[width=.2\columnwidth,height=.2\columnwidth,keepaspectratio,valign=m,margin=.0cm .05cm]{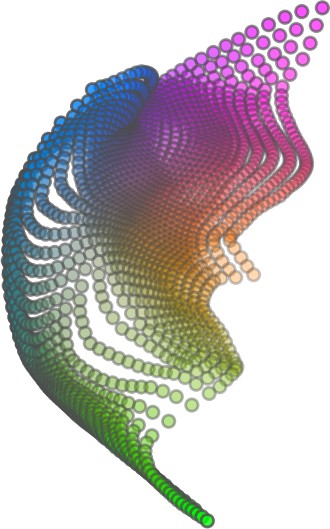}&
\includegraphics[width=.2\columnwidth,height=.2\columnwidth,keepaspectratio,valign=m,margin=.0cm .05cm]{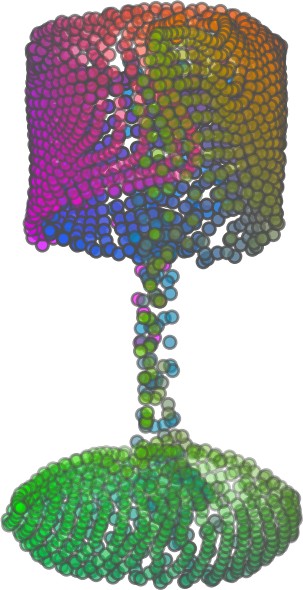}\\
\includegraphics[width=.2\columnwidth,height=.2\columnwidth,keepaspectratio,valign=m,margin=.0cm .05cm]{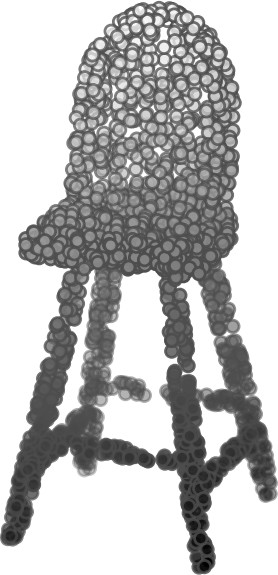}&
&
\includegraphics[width=.2\columnwidth,height=.2\columnwidth,keepaspectratio,valign=m,margin=.0cm .05cm]{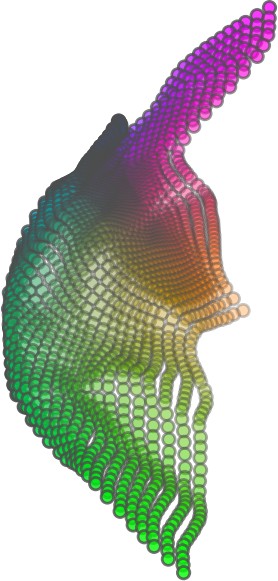}&
\includegraphics[width=.2\columnwidth,height=.2\columnwidth,keepaspectratio,valign=m,margin=.0cm .05cm]{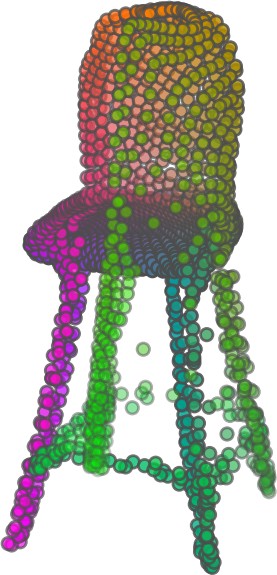}\\
\includegraphics[width=.2\columnwidth,height=.2\columnwidth,keepaspectratio,valign=m,margin=.0cm .05cm]{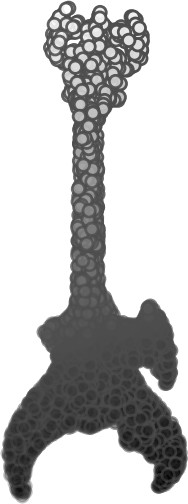}&
&
\includegraphics[width=.2\columnwidth,height=.2\columnwidth,keepaspectratio,valign=m,margin=.0cm .05cm]{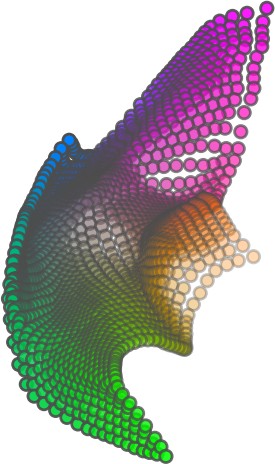}&
\includegraphics[width=.2\columnwidth,height=.2\columnwidth,keepaspectratio,valign=m,margin=.0cm .05cm]{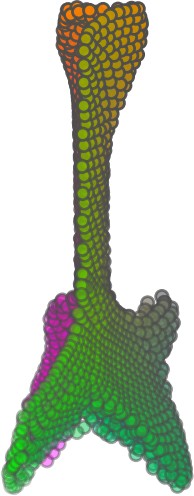}\\
\hline
\end{tabular}
\medskip
\caption{\textbf{Illustration of the two-step-folding decoding}. Column one contains the original point cloud samples from the ShapeNet dataset~\cite{wu20153d}. Column two illustrates the 2D grid points to be folded during decoding. Column three contains the output after one folding operation. Column four contains the output after two folding operations. This output is also the reconstructed point cloud. We use a color gradient to illustrate the correspondence between the 2D grid in column two and the reconstructed point clouds after folding operations in the last two columns. Best viewed in color.\vspace{-3mm}}\label{table:folding}
\end{table}

\begin{figure*}
  \centering
  \includegraphics[width=0.95\textwidth]{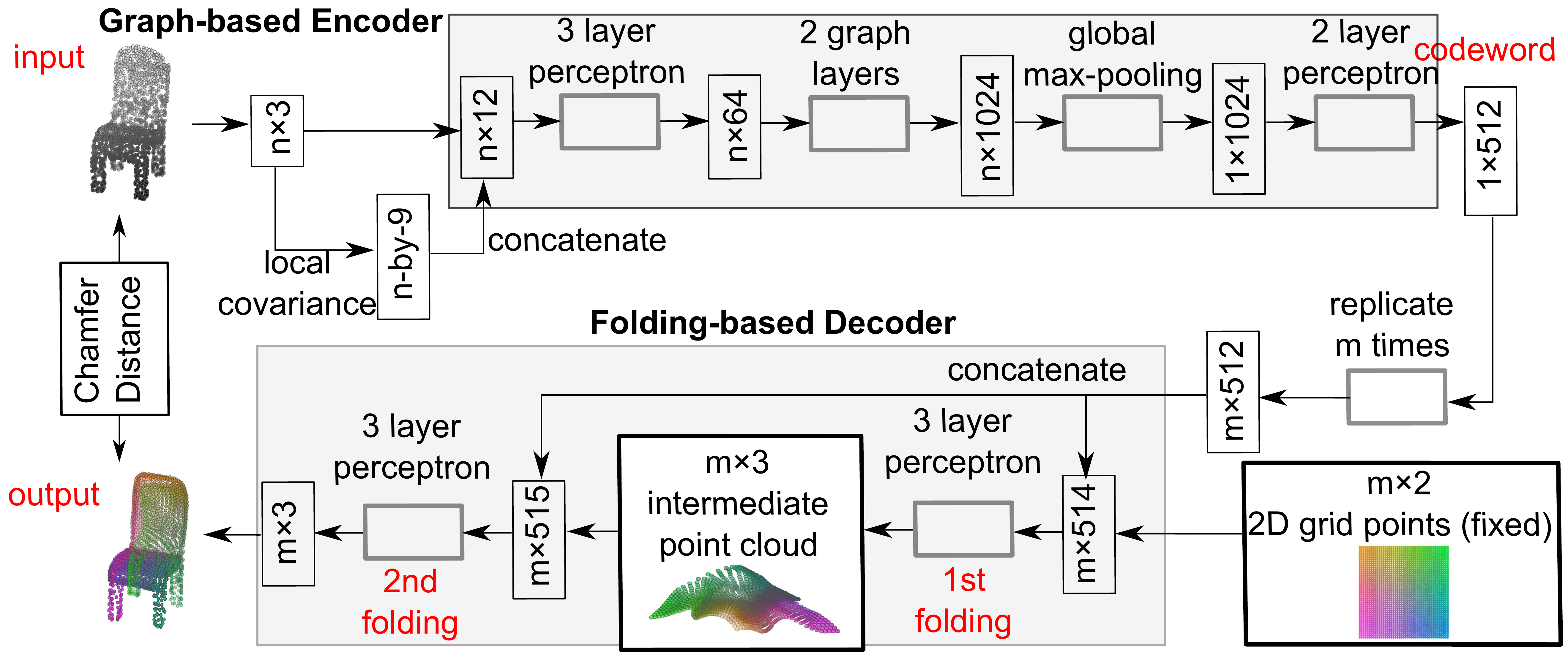}\\
  \caption{\textbf{FoldingNet Architecture}. The graph-layers are the graph-based max-pooling layers mentioned in \eqref{eqn:graph_layer} in Section~\ref{sec:enc}. The 1st and the 2nd folding are both implemented by concatenating the codeword to the feature vectors followed by a 3-layer perceptron. Each perceptron independently applies to the feature vector of a single point as in~\cite{qi2016pointnet}, i.e., applies to the rows of the $m$-by-$k$ matrix. }\vspace{-3mm} \label{fig:AE}
\end{figure*}

In this work, we focus on the emerging field of unsupervised learning for point clouds. We propose an auto-encoder (AE) that is referenced as FoldingNet. The output from the bottleneck layer in the auto-encoder is called a codeword that can be used as a high-dimensional embedding of an input point cloud. We are going to show that a 2D grid structure is not only a sampling structure for imaging, but can indeed be used to construct a point cloud through the proposed \emph{folding} operation. This is based on the observation that the 3D point clouds of our interest are obtained from object surfaces: either discretized from boundary representations in CAD/computer graphics, or sampled from line-of-sight sensors like LIDAR. Intuitively, any 3D object surface could be transformed to a 2D plane through certain operations like cutting, squeezing, and stretching. The inverse procedure is to glue those 2D point samples back onto an object surface via certain folding operations, which are initialized as 2D grid samples. As illustrated in Table~\ref{table:folding}, to reconstruct a point cloud, successive folding operations are joined to reproduce the surface structure. The points are colorized to show the correspondence between the initial 2D grid samples and the reconstructed 3D point samples. Using the folding-based method, the challenges from the irregular structure of point clouds are well addressed by directly introducing such an implicit 2D grid constraint in the decoder, which avoids the costly 3D voxelization in other works~\cite{wu2016learning}. It will be demonstrated later that the folding operations can build an arbitrary surface provided a proper codeword. Notice that when data are from volumetric format instead of 2D surfaces, a 3D grid may perform better.

Despite being strongly expressive in reconstructing point clouds, the folding operation is simple: it is started by augmenting the 2D grid points with the codeword obtained from the encoder, which is then processed through a 3-layer perceptron. The proposed decoder is simply a concatenation of two folding operations. This design makes the proposed decoder much smaller in parameter size than the fully-connected decoder proposed recently in~\cite{achlioptas2017representation}. In Section~\ref{exp:compare_FC}, we show that the number of parameters of our folding-based decoder is about 7\% of the fully connected decoder in~\cite{achlioptas2017representation}. Although the proposed decoder has a simple structure, we theoretically show in Theorem~\ref{thm:1} that this folding-based structure is universal in that one folding operation that uses only a 2-layer perceptron can already reproduce arbitrary point-cloud structure. Therefore, it is not surprising that our FoldingNet auto-encoder exploiting two consecutive folding operations can produce elaborate structures.

To show the efficiency of FoldingNet auto-encoder for unsupervised representation learning, we follow the experimental settings in~\cite{achlioptas2017representation} and test the transfer classification accuracy from ShapeNet dataset \cite{chang2015shapenet} to ModelNet dataset~\cite{wu20153d}. The FoldingNet auto-encoder is trained using ShapeNet dataset, and tested out by extracting codewords from ModelNet dataset. Then, we train a linear SVM classifier to test the discrimination effectiveness of the extracted codewords. The transfer classification accuracy is 88.4\% on the ModelNet dataset with 40 shape categories. This classification accuracy is even close to the state-of-the-art supervised training result~\cite{qi2016pointnet}. To achieve the best classification performance and least reconstruction loss, we use a graph-based encoder structure that is different from~\cite{qi2016pointnet}. This graph-based encoder is based on the idea of local feature pooling operations and is able to retrieve and propagate local structural information along the graph structure.

To intuitively interpret our network design: we want to impose a ``virtual force'' to deform/cut/stretch a 2D grid lattice onto a 3D object surface, and such a deformation force should be influenced or regulated by interconnections induced by the lattice neighborhood. 
Since the intermediate folding steps in the decoder and the training process can be illustrated by reconstructed points, the gradual change of the folding forces can be visualized.

Now we summarize our contributions in this work:
\begin{itemize}[noitemsep,nolistsep]
\item We train an end-to-end deep auto-encoder that consumes unordered point clouds directly.
\item We propose a new decoding operation called folding and theoretically show it is universal in point cloud reconstruction, while providing orders to reconstructed points as a unique byproduct than other methods.
\item We show by experiments on major datasets that folding can achieve higher classification accuracy than other unsupervised methods.
\end{itemize}

%-------------------------------------------------------------------------
\subsection{Related works}\label{sec:related_works}

Applications of learning on point clouds include shape completion and recognition \cite{wu20153d}, unmanned autonomous vehicles \cite{maturana20153d}, 3D object detection, recognition and classification \cite{song2016deep, socher2012convolutional,pang2015fast,chen20153d,qi2016pointnet,li2016fpnn,wahl2003surflet}, contour detection \cite{hackel2016contour}, layout inference \cite{geiger2015joint}, scene labeling \cite{lai2014unsupervised}, category discovery \cite{zhang2013unsupervised}, point classification, dense labeling and segmentation \cite{kim20133d,maturana2015voxnet,wang2013label,christoph2014object,dohan2015learning,hackel2016fast,qi2016pointnet,boulch2017unstructured, yi2016scalable,huang2016point,wang20173d},

Most deep neural networks designed for 3D point clouds are based on the idea of partitioning the 3D space into regular voxels and extending 2D CNNs to voxels, such as \cite{maturana2015voxnet,dai2017scannet,brock2016generative}, including the the work on 3D generative adversarial network \cite{wu2016learning}. The main problem of voxel-based networks is the fast growth of neural-network size with the increasing spatial resolution. Some other options include
octree-based \cite{riegler2016octnet} and kd-tree-based \cite{klokov2017escape} neural networks. Recently, it is shown that neural networks based on purely 3D point representations \cite{qi2016pointnet,ravanbakhsh2016deep,achlioptas2017representation,qi2017pointnet++} work quite efficiently for point clouds. The point-based neural networks can reduce the overhead of converting point clouds into other data formats (such as octrees and voxels), and in the meantime avoid the information loss due to the conversion.

The only work that we are aware of on end-to-end deep auto-encoder that directly handles point clouds is \cite{achlioptas2017representation}. The AE designed in \cite{achlioptas2017representation} is for the purpose of extracting features for generative networks. To encode, it sorts the 3D points using the lexicographic order and applies a 1D CNN on the point sequence. To decode, it applies a three-layer fully connected network. This simple structure turns out to outperform all existing unsupervised works on representation extraction of point clouds in terms of the transfer classification accuracy from the ShapeNet dataset to the ModelNet dataset \cite{achlioptas2017representation}. Our method, which has a graph-based encoder and a folding-based decoder, outperforms this method in transfer classification accuracy on the ModelNet40 dataset \cite{achlioptas2017representation}. Moreover, compared to \cite{achlioptas2017representation}, our AE design is more interpretable: the encoder learns the local shape information and combines information by max-pooling on a nearest-neighbor graph, and the decoder learns a ``\emph{force}'' to fold a two-dimensional grid twice in order to warp the grid into the shape of the point cloud, using the information obtained by the encoder. Another closely related work reconstructs a point set from a 2D image \cite{fan2016point}. Although the deconvolution network in \cite{fan2016point} requires a 2D image as side information, we find it useful as another implementation of our folding operation. We compare FoldingNet with the deconvolution-based folding and show that FoldingNet performs slightly better in reconstruction error with fewer parameters (see Supplementary Section~\ref{sec:devonvolution}).

It is hard for purely point-based neural networks to extract local neighborhood structure around points, i.e., features of neighboring points instead of individual ones. Some attempts for this are made in \cite{qi2017pointnet++,achlioptas2017representation}. In this work, we exploit local neighborhood features using a graph-based framework. Deep learning on graph-structured data is not a new idea. There are tremendous amount of works on applying deep learning onto irregular data such as graphs and point sets \cite{ravanbakhsh2016deep,henaff2015deep,duvenaud2015convolutional,vialatte2016generalizing,bruna2013spectral,masci2015geodesic,monti2016geometric,bronstein2017geometric,defferrard2016convolutional,niepert2016learning,edwards2016graph,kipf2016semi,hechtlinger2017generalization,levie2017cayleynets,atwood2016diffusion,simonovsky2017dynamic,zaheer2017deep}. Although using graphs as a processing framework for deep learning on point clouds is a natural idea, only several seminal works made attempts in this direction~\cite{bronstein2017geometric,monti2016geometric,simonovsky2017dynamic}. These works try to generalize the convolution operations from 2D images to graphs. However, since it is hard to define convolution operations on graphs, we use a simple graph-based neural network layer that is different from previous works: we construct the K-nearest neighbor graph (K-NNG) and repeatedly conduct the max-pooling operations in each node's neighborhood. It generalizes the global max-pooling operation proposed in~\cite{qi2016pointnet} in that the max-pooling is only applied to each local neighborhood to generate local data signatures. Compared to the above graph based convolution networks, our design is simpler and computationally efficient as in~\cite{qi2016pointnet}. K-NNGs are also used in other applications of point clouds without the deep learning framework such as surface detection, 3D object recognition, 3D object segmentation and compression~\cite{golovinskiy2009shape, strom2010graph, thanou2016graph}.

The folding operation that reconstructs a surface from a 2D grid essentially establishes a mapping from a 2D regular domain to a 3D point cloud. A natural question to ask is whether we can parameterize 3D points with compatible meshes that are not necessarily regular grids, such as \emph{cross-parametrization} \cite{kraevoy2004cross}. From Table~\ref{table:visual_training}, it seems that FoldingNet can learn to generate ``cuts'' on the 2D grid and generate surfaces that are not even topologically equivalent to a 2D grid, and hence make the 2D grid representation universal to some extent. Nonetheless, the reconstructed points may still have genus-wise distortions when the original surface is too complex. For example, in Table \ref{table:visual_training}, see the missing winglets on the reconstructed plane and the missing holes on the back of the reconstructed chair. To recover those finer details might require more input point samples and more complex encoder/decoder networks.
Another method to learn the surface embedding is to learn a metric alignment layer as in \cite{ezuz2017gwcnn}, which may require computationally intensive internal optimization during training.
\subsection{Preliminaries and Notation}

We will often denote the point set by $S$. We use bold lower-case letters to represent vectors, such as $\mathbf{x}$, and use bold upper-case letters to represent matrices, such as $\mathbf{A}$. The codeword is always represented by $\boldsymbol{\theta}$. We call a matrix $m$-by-$n$ or $m\times n$ if it has $m$ rows and $n$ columns.

\section{FoldingNet Auto-encoder on Point Clouds}

Now we propose the FoldingNet deep auto-encoder. The structure of the auto-encoder is shown in Figure~\ref{fig:AE}. The input to the encoder is an $n$-by-3 matrix. Each row of the matrix is composed of the 3D position $(x,y,z)$. The output is an $m$-by-3 matrix, representing the reconstructed point positions. The number of reconstructed points $m$ is not necessarily the same as $n$. Suppose the input contains the point set $S$ and the reconstructed point set is the set $\widehat{S}$. Then, the reconstruction error for $\widehat{S}$ is computed using a layer defined as the (extended) Chamfer distance,
\begin{equation}\label{eqn:chamfer}
\begin{split}
	d_{CH}(S,\widehat{S}) = \max\left\{\frac{1}{|S|}\sum\limits_{\mathbf{x}\in S}\min_{\widehat{\mathbf{x}}\in \widehat{S}}\|\mathbf{x}-\widehat{\mathbf{x}}\|_2, \right.\\
    \left.\frac{1}{|\widehat{S}|}\sum\limits_{\widehat{\mathbf{x}}\in \widehat{S}}\min_{\mathbf{x}\in S}\|\widehat{\mathbf{x}} - \mathbf{x}\|_2\right\}.
\end{split}
\end{equation}
The term $\min_{\widehat{\mathbf{x}}\in \widehat{S}}\|\mathbf{x}-\widehat{\mathbf{x}}\|_2$ enforces that any 3D point $\mathbf{x}$ in the original point cloud has a matching 3D point $\widehat{\mathbf{x}}$ in the reconstructed point cloud, and the term $\min_{\mathbf{x}\in S}\|\widehat{\mathbf{x}} - \mathbf{x}\|_2$ enforces the matching vice versa. The $\text{max}$ operation enforces that the distance from $S$ to $\widehat{S}$ and the distance vice versa have to be small simultaneously. The encoder computes a representation (codeword) of each input point cloud and the decoder reconstructs the point cloud using this codeword. In our experiments, the codeword length is set as 512 in accordance with \cite{achlioptas2017representation}.

\subsection{Graph-based Encoder Architecture}\label{sec:enc}

The graph-based encoder follows a similar design in \cite{shen2018kcnet} which focuses on supervised learning using point cloud neighborhood graphs. The encoder is a concatenation of multi-layer perceptrons (MLP) and graph-based max-pooling layers. The graph is the K-NNG constructed from the 3D positions of the nodes in the input point set. In experiments, we choose $K=16$. First, for every single point $v$, we compute its local covariance matrix of size 3-by-3 and vectorize it to size 1-by-9. The local covariance of $v$ is computed using the 3D positions of the points that are one-hop neighbors of $v$ (including $v$) in the K-NNG. We concatenate the matrix of point positions with size $n$-by-3 and the local covariances for all points of size $n$-by-9 into a matrix of size $n$-by-12 and input them to a 3-layer perceptron. The perceptron is applied in parallel to each row of the input matrix of size $n$-by-12. It can be viewed as a per-point function on each 3D point. The output of the perceptron is fed to two consecutive graph layers, where each layer applies max-pooling to the neighborhood of each node. More specifically, suppose the K-NN graph has adjacency matrix $\mathbf{A}$ and the input matrix to the graph layer is $\mathbf{X}$. Then, the output matrix is\vspace{-3mm}
\begin{equation}\label{eqn:graph_layer}
\mathbf{Y}=\mathbf{A}_\text{max}(\mathbf{X})\mathbf{K},
\end{equation}
where $\mathbf{K}$ is a feature mapping matrix, and the ($i$,$j$)-th entry of the matrix $\mathbf{A}_\text{max}(\mathbf{X})$ is
\begin{equation}\label{eqn:local_max}
(\mathbf{A}_\text{max}(\mathbf{X}))_{ij} = \text{ReLU}(\max_{k\in\mathcal{N}(i)}x_{kj}).
\end{equation}
The local max-pooling operation $\max_{k\in\mathcal{N}(i)}$ in \eqref{eqn:local_max} essentially computes a local signature based on the graph structure. This signature can represent the (aggregated) topology information of the local neighborhood. Through concatenations of the graph-based max-pooling layers, the network propagates the topology information into larger areas.

\subsection{Folding-based Decoder Architecture}\label{sec:dec}

The proposed decoder uses two consecutive 3-layer perceptrons to warp a fixed 2D grid into the shape of the input point cloud. The input codeword is obtained from the graph-based encoder. Before we feed the codeword into the decoder, we replicate it $m$ times and concatenate the $m$-by-512 matrix with an $m$-by-2 matrix that contains the $m$ grid points on a square centered at the origin. The result of the concatenation is a matrix of size $m$-by-$514$. The matrix is processed row-wise by a 3-layer perceptron and the output is a matrix of size $m$-by-3. After that, we again concatenate the replicated codewords to the $m$-by-3 output and feed it into a 3-layer perceptron. This output is the reconstructed point cloud. The parameter $n$ is set as per the input point cloud size, e.g. n = 2048 in our experiments, which is the same as \cite{achlioptas2017representation}.We choose $m$ grid points in a square, so $m$ is chosen as $2025$ which is the closest square number to 2048.

\begin{definition}\label{def:folding}
We call the concatenation of replicated codewords to low-dimensional grid points, followed by a point-wise MLP a \emph{folding} operation.
\end{definition}
The folding operation essentially forms a universal 2D-to-3D mapping. To intuitively see why this folding operation is a universal 2D-to-3D mapping, denote the input 2D grid points by the matrix $\mathbf{U}$. Each row of $\mathbf{U}$ is a two-dimensional grid point. Denote the $i$-th row of $\mathbf{U}$ by $\mathbf{u}_i$ and the codeword output from the encoder by $\boldsymbol{\theta}$. Then, after concatenation, the $i$-th row of the input matrix to the MLP is $[\mathbf{u}_i,\boldsymbol{\theta}]$. Since the MLP is applied in parallel to each row of the input matrix, the $i$-th row of the output matrix can be written as $f([\mathbf{u}_i,\boldsymbol{\theta}])$, where $f$ indicates the function conducted by the MLP. This function can be viewed as a parameterized high-dimensional function with the codeword $\boldsymbol{\theta}$ being a parameter to guide the structure of the function (the folding operation). Since MLPs are good at approximating non-linear functions, they can perform elaborate folding operations on the 2D grids. The high-dimensional codeword essentially stores the \emph{force} that is needed to do the folding, which makes the folding operation more diverse.

The proposed decoder has two successive folding operations. The first one folds the 2D grid to 3D space, and the second one folds inside the 3D space. We show the outputs after these two folding operations in Table~\ref{table:folding}. From column C and column D in Table~\ref{table:folding}, we can see that each folding operation conducts a relatively simple operation, and the composition of the two folding operations can produce quite elaborate surface shapes. Although the first folding seems simpler than the second one, together they lead to substantial changes in the final output. More successive folding operations can be applied if more elaborate surface shapes are required. More variations of the decoder including changes of grid dimensions and the number of folding operations can be found in Supplementary Section~\ref{sec:sup:decoders}.

\begin{table*}
\centering
\begin{tabular}{cccccccc}
\hline
Input &5K iters&10K iters&20K iters&40K iters&100K iters&500K iters&4M iters\\
\hline \\[-2.5ex]
%\hline
\includegraphics[height=0.080000\textwidth]{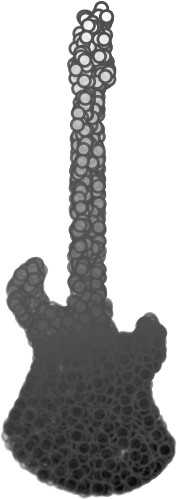}&
\includegraphics[height=0.080000\textwidth]{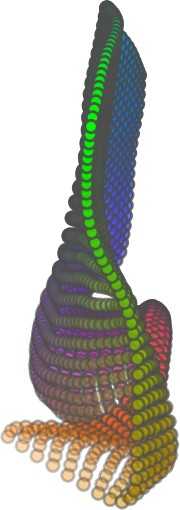}&
\includegraphics[height=0.080000\textwidth]{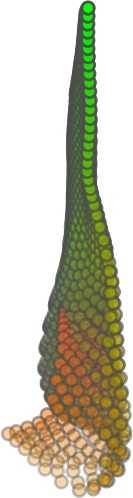}&
\includegraphics[height=0.080000\textwidth]{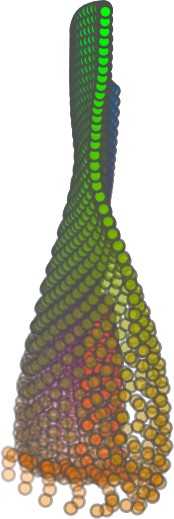}&
\includegraphics[height=0.080000\textwidth]{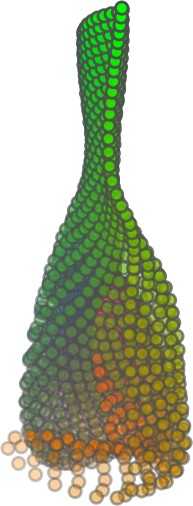}&
\includegraphics[height=0.080000\textwidth]{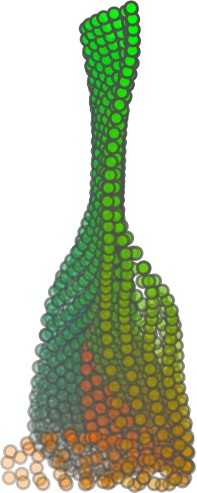}&
\includegraphics[height=0.080000\textwidth]{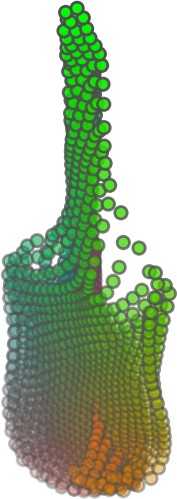}&
\includegraphics[height=0.080000\textwidth]{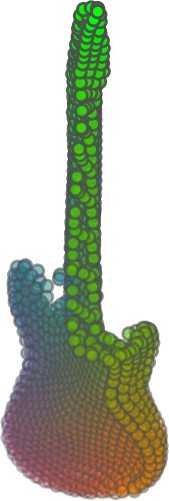}\\
%\hline
\includegraphics[height=0.060000\textwidth]{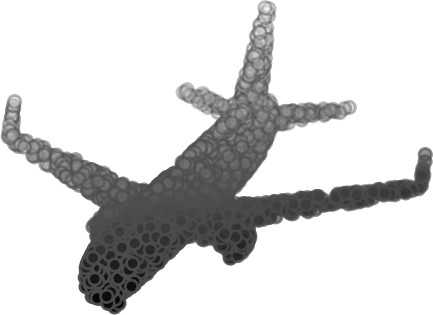}&
\includegraphics[height=0.060000\textwidth]{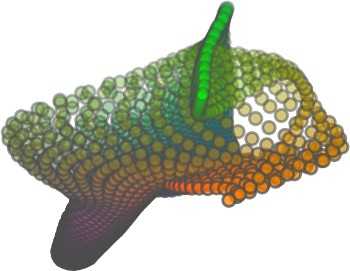}&
\includegraphics[height=0.060000\textwidth]{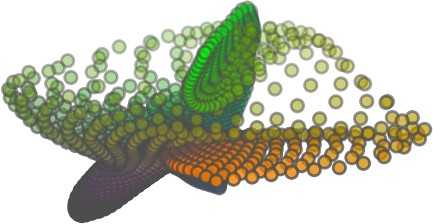}&
\includegraphics[height=0.060000\textwidth]{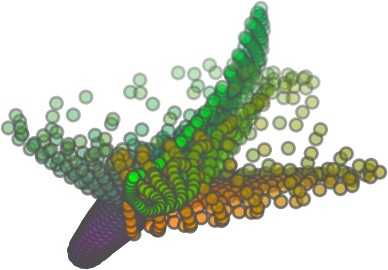}&
\includegraphics[height=0.060000\textwidth]{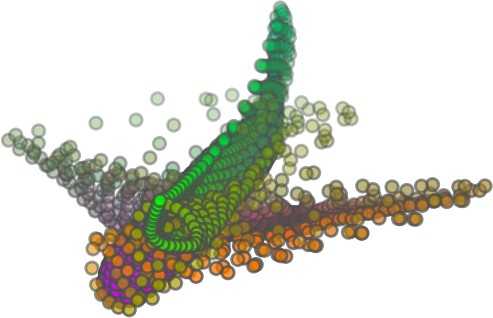}&
\includegraphics[height=0.060000\textwidth]{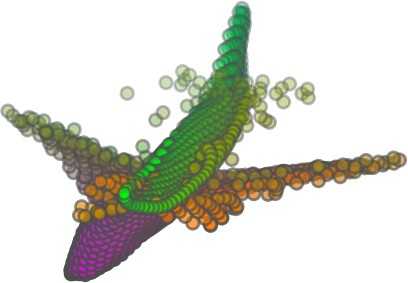}&
\includegraphics[height=0.060000\textwidth]{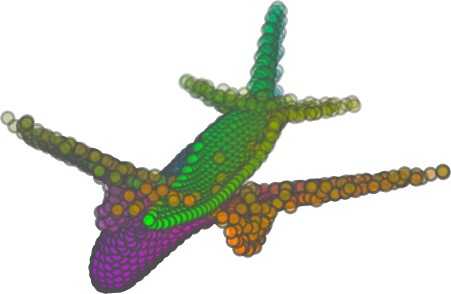}&
\includegraphics[height=0.060000\textwidth]{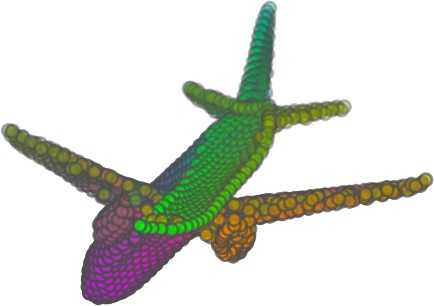}\\
%\hline
\includegraphics[height=0.060000\textwidth]{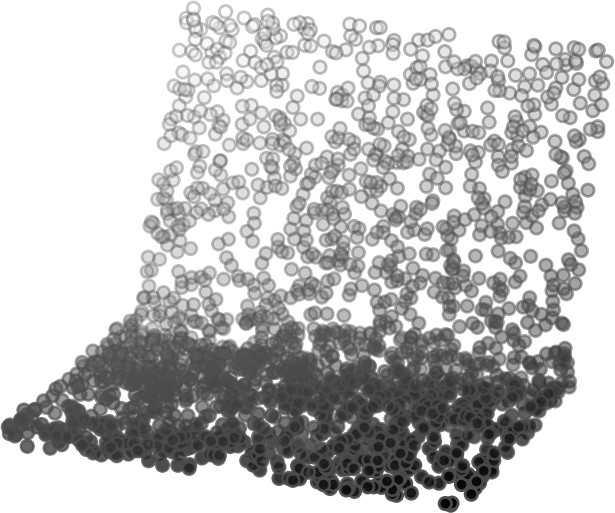}&
\includegraphics[height=0.060000\textwidth]{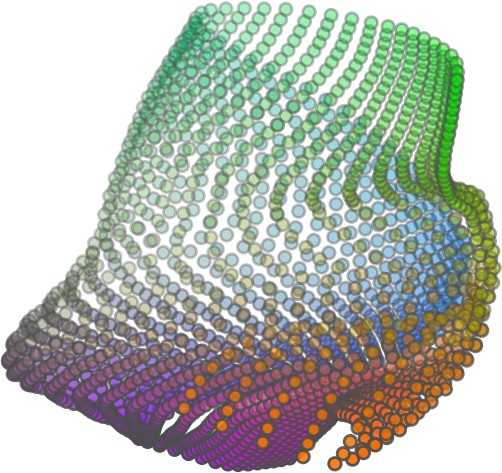}&
\includegraphics[height=0.060000\textwidth]{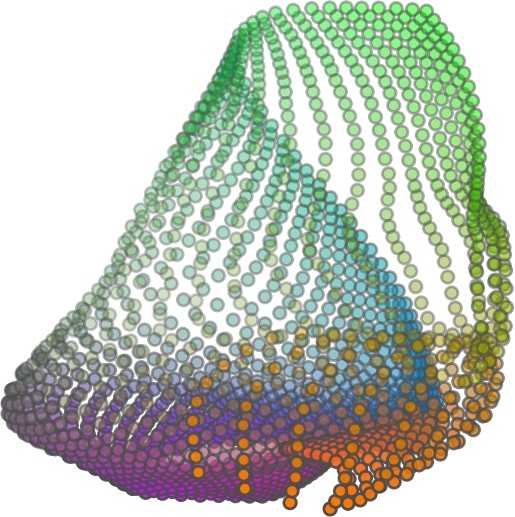}&
\includegraphics[height=0.060000\textwidth]{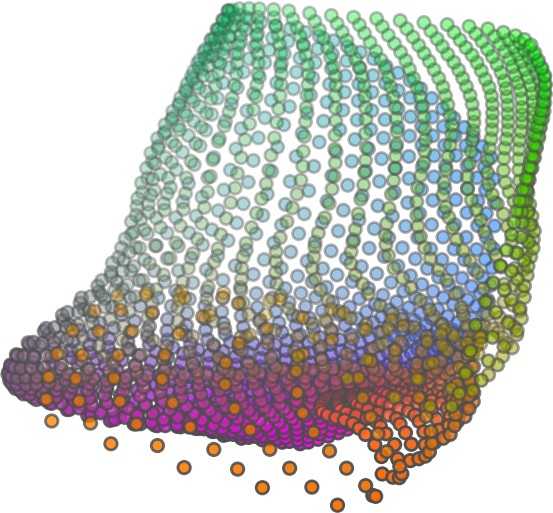}&
\includegraphics[height=0.060000\textwidth]{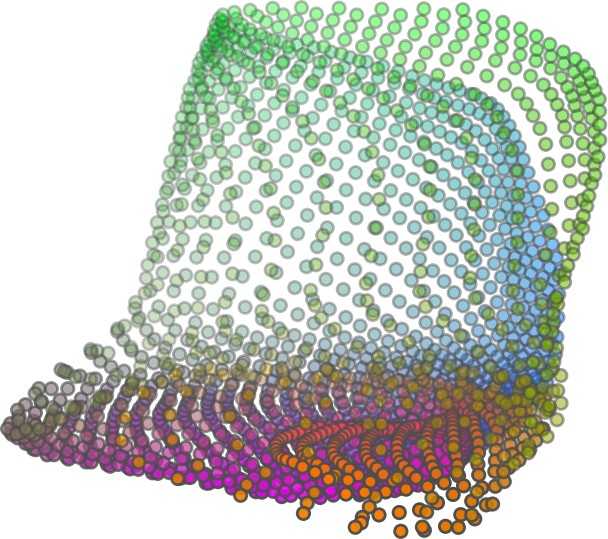}&
\includegraphics[height=0.060000\textwidth]{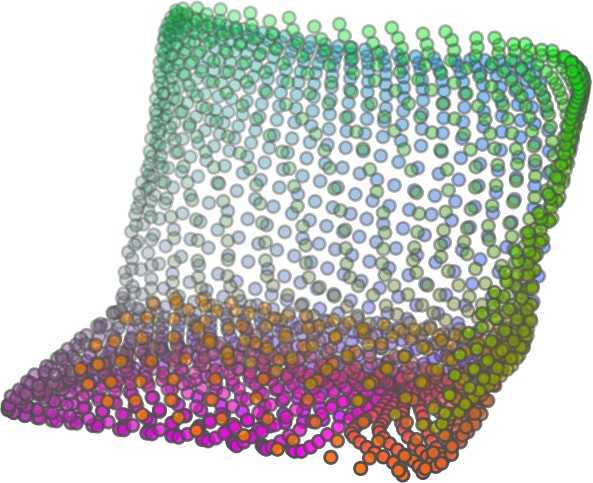}&
\includegraphics[height=0.060000\textwidth]{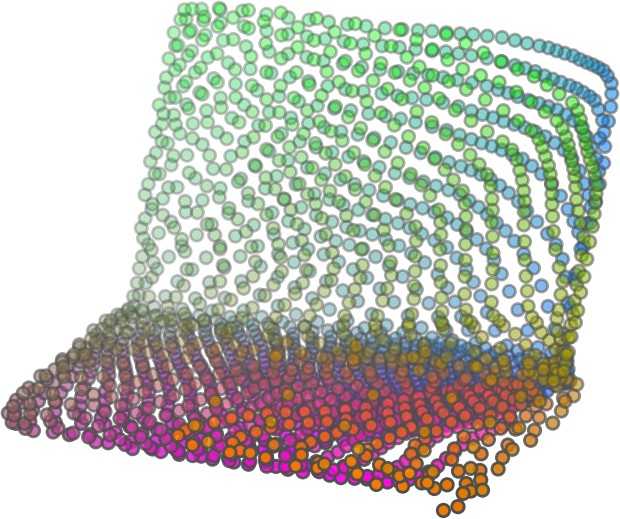}&
\includegraphics[height=0.060000\textwidth]{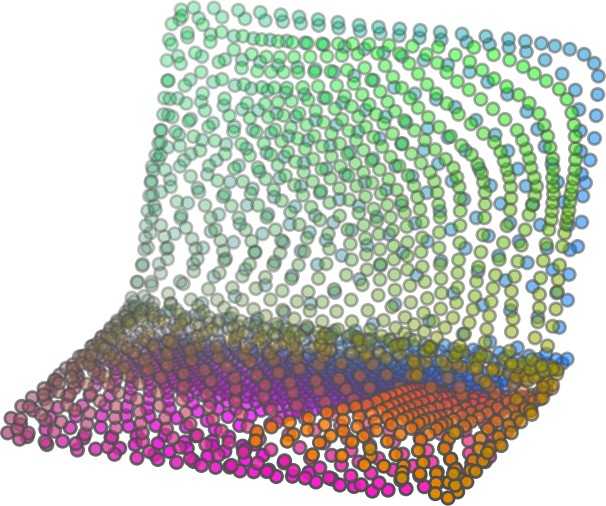}\\
%\hline
\includegraphics[height=0.05000\textwidth]{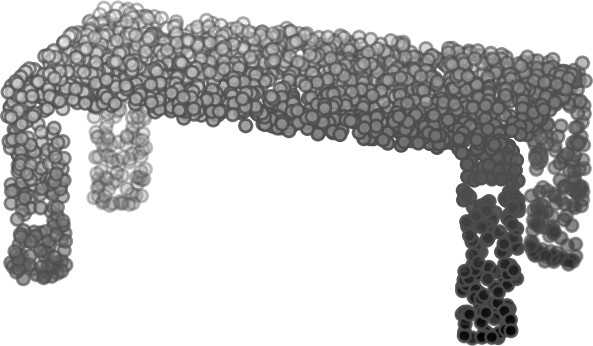}&
\includegraphics[height=0.05000\textwidth]{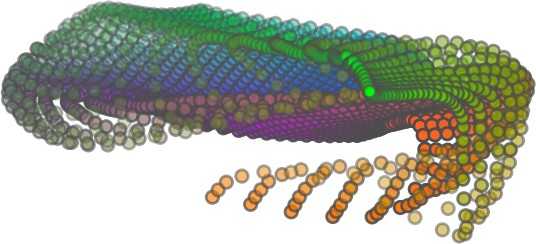}&
\includegraphics[height=0.05000\textwidth]{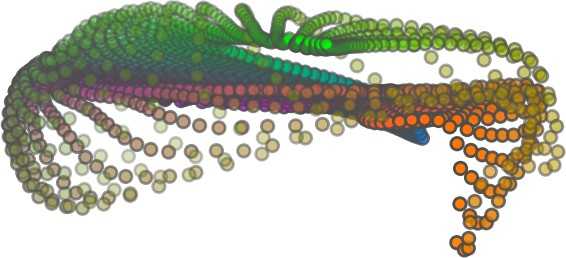}&
\includegraphics[height=0.05000\textwidth]{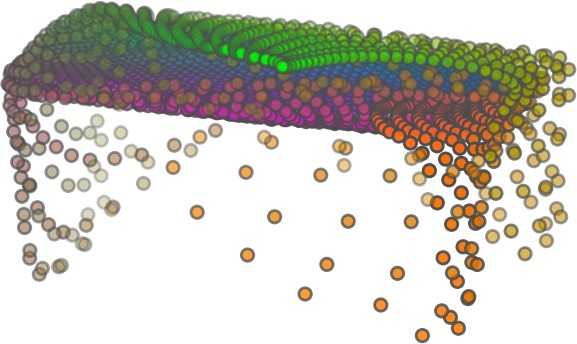}&
\includegraphics[height=0.05000\textwidth]{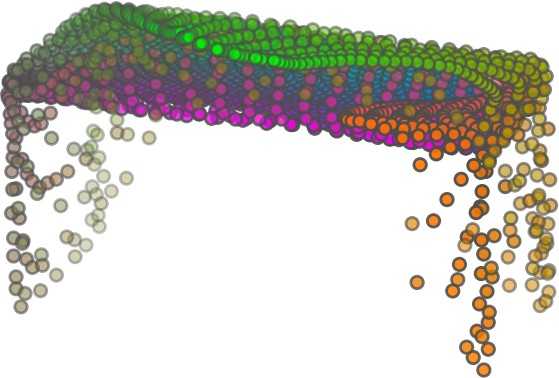}&
\includegraphics[height=0.05000\textwidth]{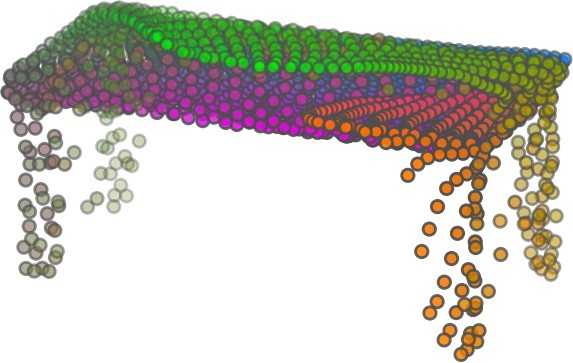}&
\includegraphics[height=0.05000\textwidth]{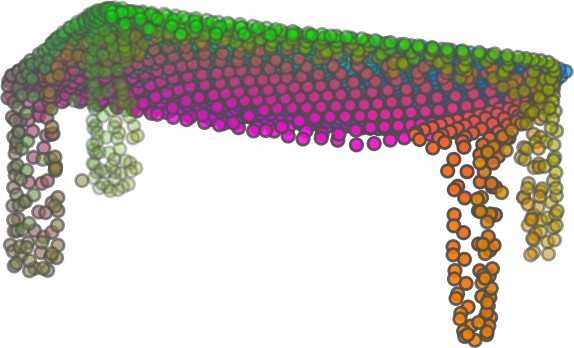}&
\includegraphics[height=0.05000\textwidth]{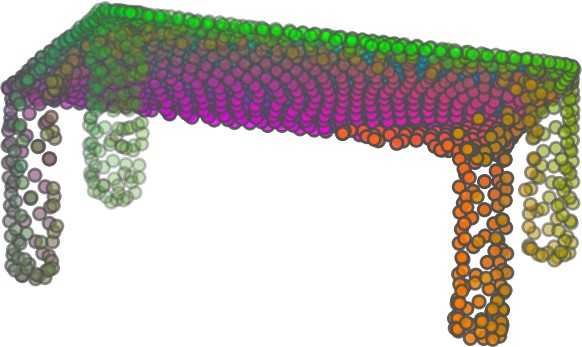}\\
%\hline
\includegraphics[height=0.070000\textwidth]{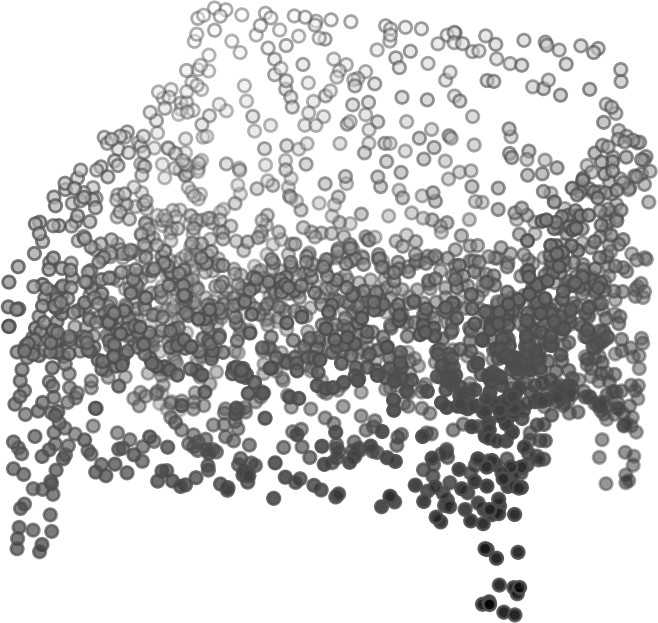}&
\includegraphics[height=0.070000\textwidth]{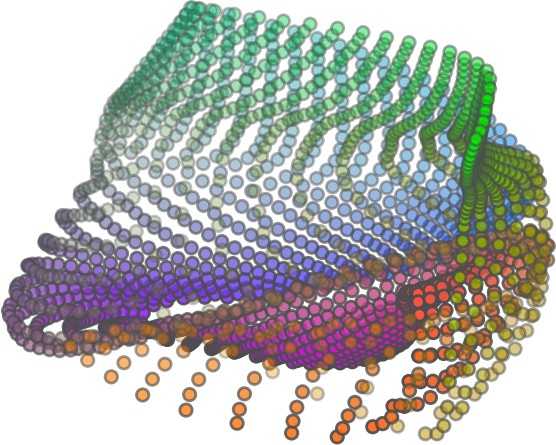}&
\includegraphics[height=0.070000\textwidth]{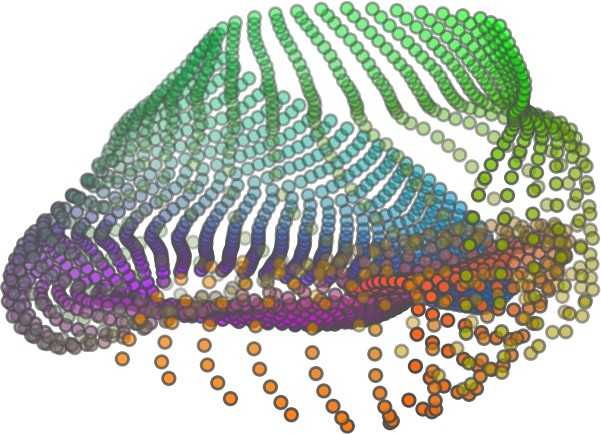}&
\includegraphics[height=0.070000\textwidth]{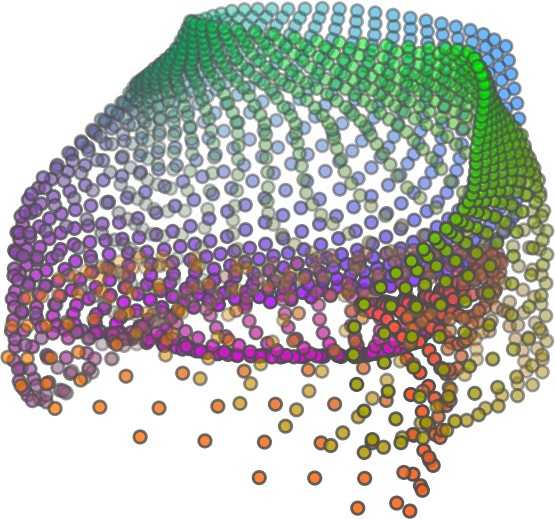}&
\includegraphics[height=0.070000\textwidth]{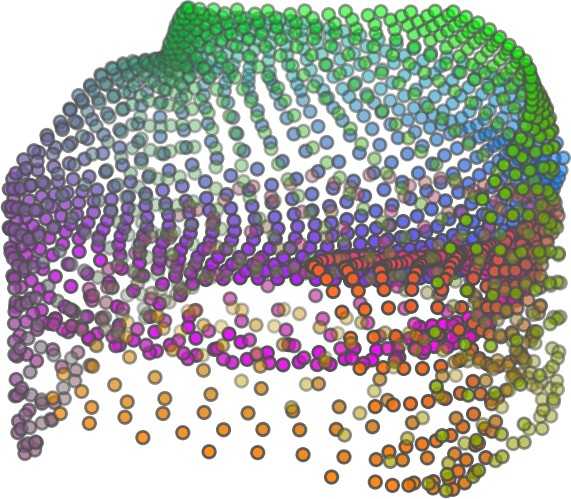}&
\includegraphics[height=0.070000\textwidth]{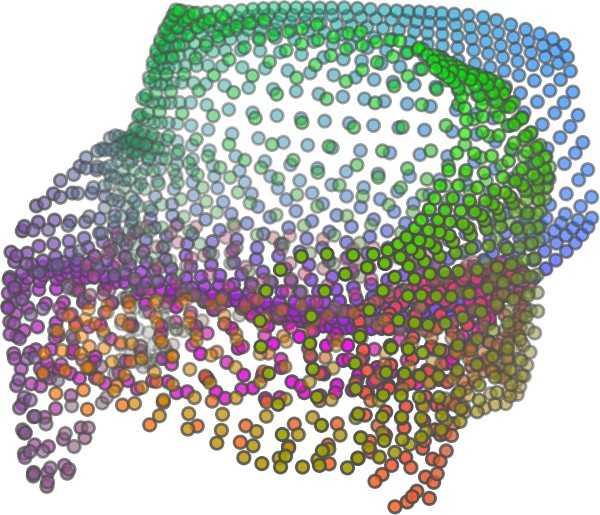}&
\includegraphics[height=0.070000\textwidth]{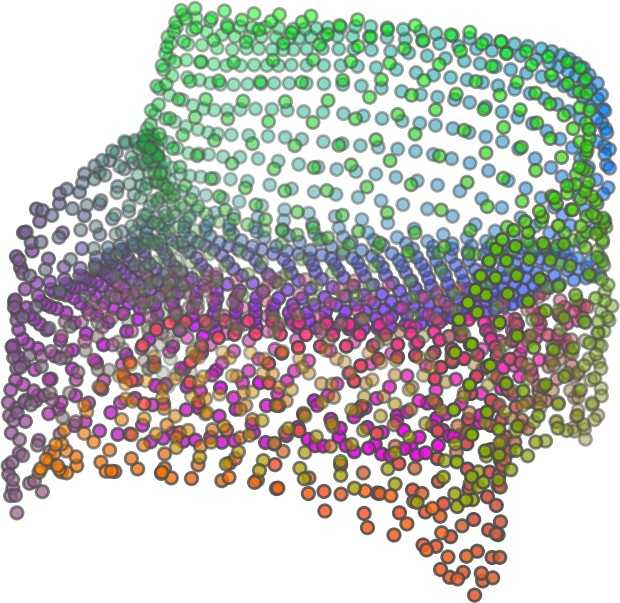}&
\includegraphics[height=0.070000\textwidth]{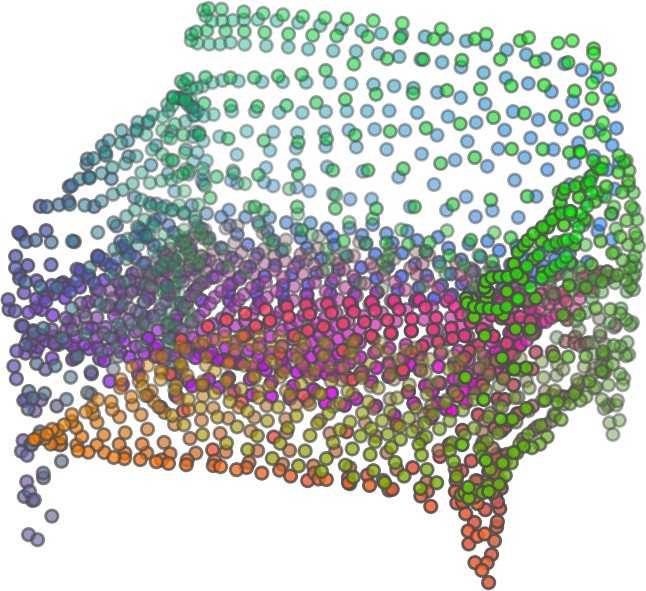}\\
%\hline
\includegraphics[height=0.070000\textwidth]{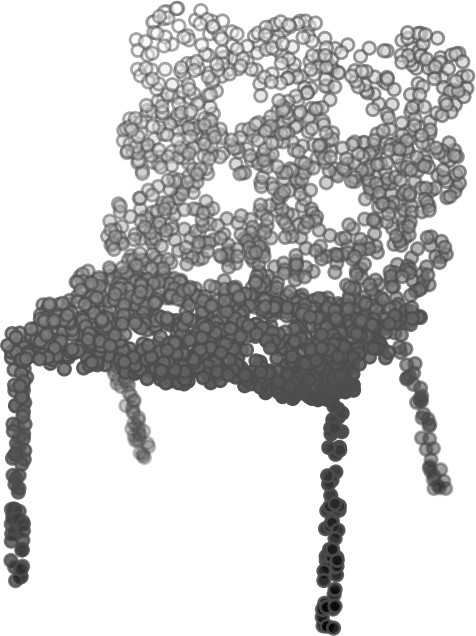}&
\includegraphics[height=0.070000\textwidth]{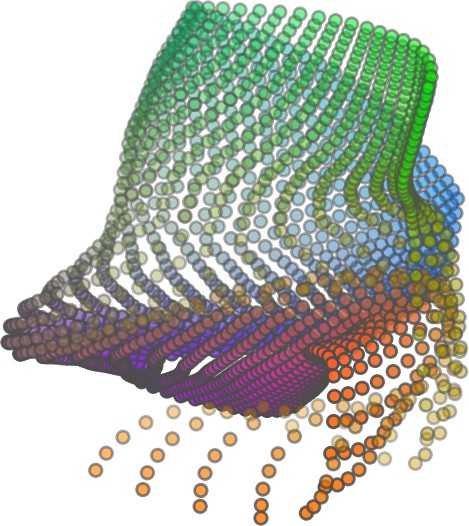}&
\includegraphics[height=0.070000\textwidth]{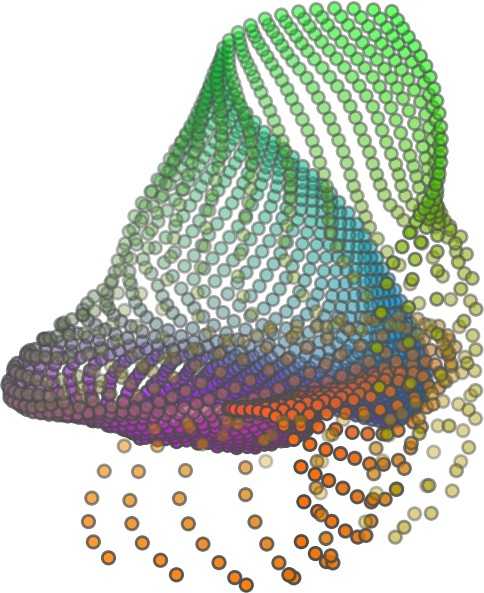}&
\includegraphics[height=0.070000\textwidth]{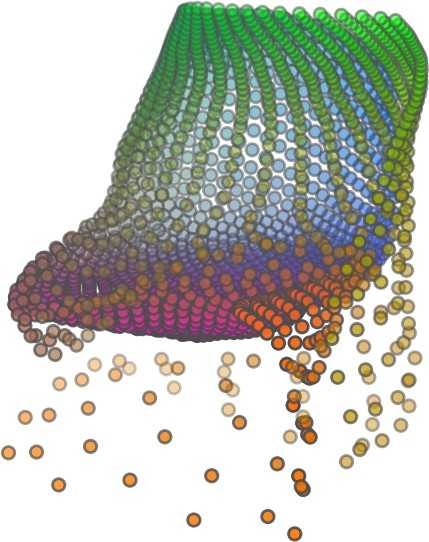}&
\includegraphics[height=0.070000\textwidth]{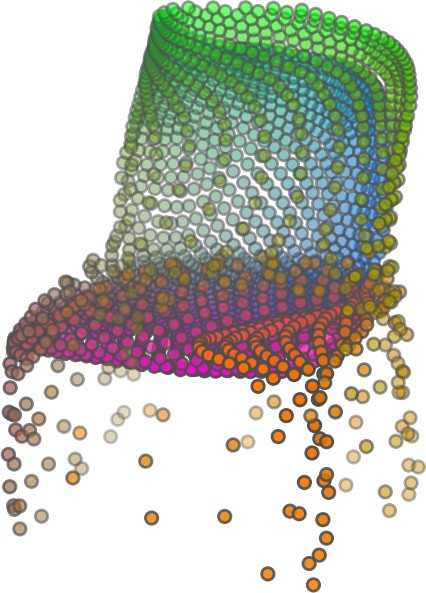}&
\includegraphics[height=0.070000\textwidth]{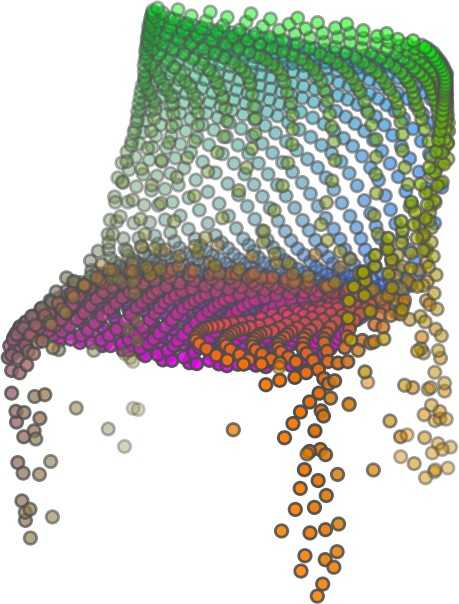}&
\includegraphics[height=0.070000\textwidth]{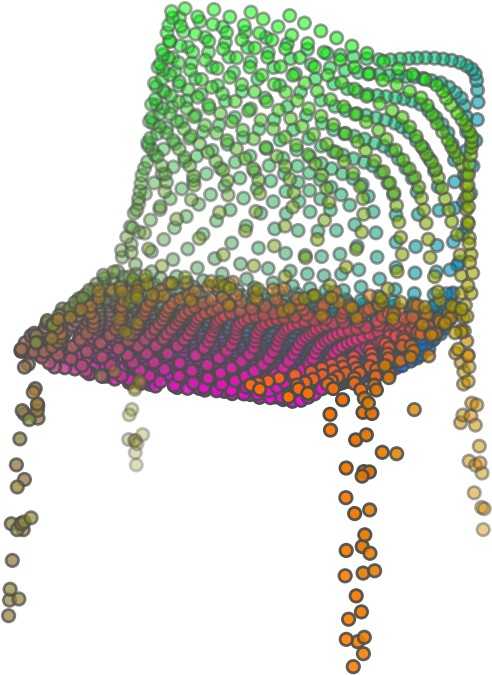}&
\includegraphics[height=0.070000\textwidth]{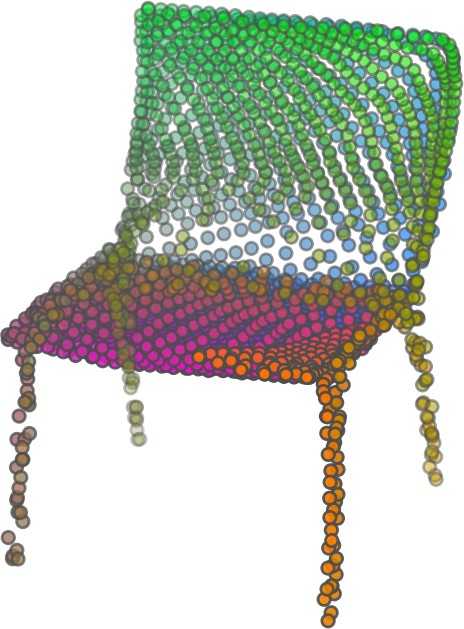}\\
\hline
\end{tabular}\vspace{1mm}
\caption{Illustration of the training process. Random 2D manifolds gradually transform into the surfaces of point clouds.}
%\vspace{-2mm}
\label{table:visual_training}
\end{table*}
\section{Theoretical Analysis}\label{sec:theory}

\begin{theorem}\label{thm:perm_inv}
The proposed encoder structure is permutation invariant, i.e., if the rows of the input point cloud matrix are permuted, the codeword remains unchanged.
\end{theorem}
\begin{proof}
See Supplementary Section~\ref{sec:permutation}.
\end{proof}

Then, we state a theorem about the universality of the proposed folding-based decoder. It shows the existence of a folding-based decoder such that by changing the codeword $\boldsymbol{\theta}$, the output can be an arbitrary point cloud.

\begin{theorem}\label{thm:1}
There exists a 2-layer perceptron that can reconstruct arbitrary point clouds from a 2-dimensional grid using the \emph{folding} operation.

More specifically, suppose the input is a matrix $\mathbf{U}$ of size $m$-by-2 such that each row of $\mathbf{U}$ is the 2D position of a point on a 2-dimensional grid of size $m$. Then, there exists an explicit construction of a 2-layer perceptron (with hand-crafted coefficients) such that for any arbitrary 3D point cloud matrix $\mathbf{S}$ of size $m$-by-3 (where each row of $\mathbf{S}$ is the $(x,y,z)$ position of a point in the point cloud), there exists a codeword vector $\boldsymbol{\theta}$ such that if we concatenate $\boldsymbol{\theta}$ to each row of $\mathbf{U}$ and apply the 2-layer perceptron in parallel to each row of the matrix after concatenation, we obtain the point cloud matrix $\mathbf{S}$ from the output of the perceptron.
\end{theorem}

\begin{proof}[Proof in sketch]
The full proof is in Supplementary Section \ref{sec:proof}. In the proof, we show the existence by explicitly constructing a 2-layer perceptron that satisfies the stated properties. The main idea is to show that in the worst case, the points in the 2D grid functions as a selective logic gate to map the 2D points in the 2D grid to the corresponding 3D points in the point cloud.
\end{proof}

Notice that the above proof is just an existence-based one to show that our decoder structure is universal. It \emph{does not} indicate what happens in reality inside the FoldingNet auto-encoder. The theoretically constructed decoder requires $3m$ hidden units while in reality, the size of the decoder that we use is much smaller. Moreover, the construction in Theorem~\ref{thm:1} leads to a lossless reconstruction of the point cloud, while the FoldingNet auto-encoder only achieves lossy reconstruction. However, the above theorem can indeed guarantee that the proposed decoding operation (i.e., concatenating the codewords to the 2-dimensional grid points and processing each row using a perceptron) is legitimate because in the worst case there exists a folding-based neural network with hand-crafted edge weights that can reconstruct arbitrary point clouds. In reality, a good parameterization of the proposed decoder with suitable training leads to better performance.

\section{Experimental Results}

\subsection{Visualization of the Training Process}\label{exp:training}

It might not be straightforward to see how the decoder folds the 2D grid into the surface of a 3D point cloud. Therefore, we include an illustration of the training process to show how a random 2D manifold obtained by the initial random folding gradually turns into a meaningful point cloud. The auto-encoder is a single FoldingNet trained using the ShapeNet part dataset \cite{yi2016scalable} which contains 16 categories of the ShapeNet dataset. We trained the FoldingNet using ADAM with an initial learning rate 0.0001, batch size 1, momentum 0.9, momentum2 0.999, and weight decay $1\mathrm{e}{-6}$, for $4\times10^6$ iterations (i.e., 330 epochs). The reconstructed point clouds of several models after different numbers of training iterations are reported in Table~\ref{table:visual_training}. From the training process, we see that an initial random 2D manifold can be warped/cut/squeezed/stretched/attached to form the point cloud surface in various ways.

\subsection{Point Cloud Interpolation}

\begin{table*}
	\begin{tabular}{cccccccc}
		\hline
		Source & \multicolumn{6}{c}{Interpolations} &Target\\
		\hline\\[-2.5ex]
		%\hline
		\includegraphics[height=0.064000\textwidth]{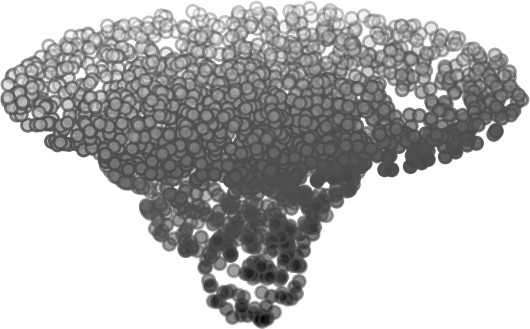}&
		\includegraphics[height=0.064000\textwidth]{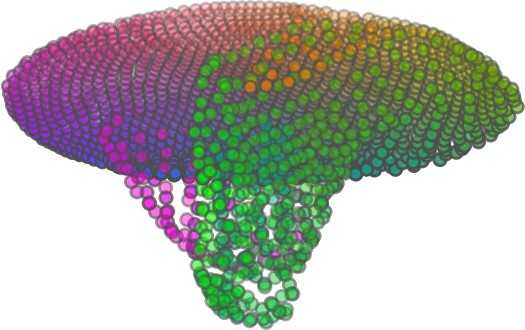}&
		\includegraphics[height=0.064000\textwidth]{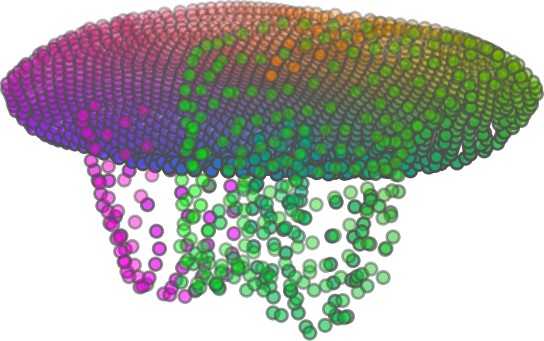}&
		\includegraphics[height=0.064000\textwidth]{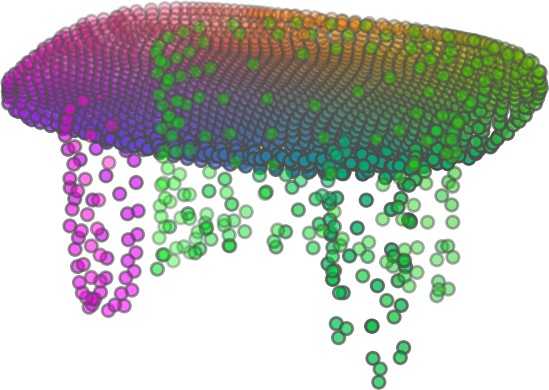}&
		\includegraphics[height=0.064000\textwidth]{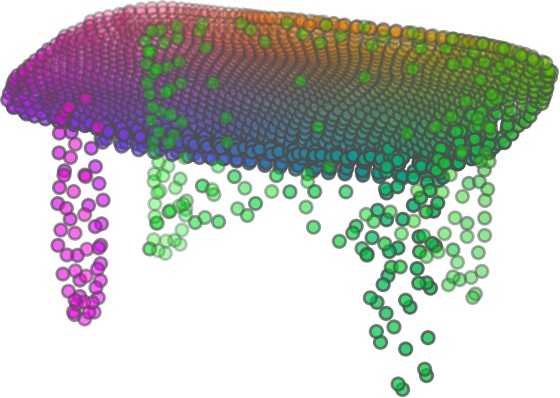}&
		\includegraphics[height=0.064000\textwidth]{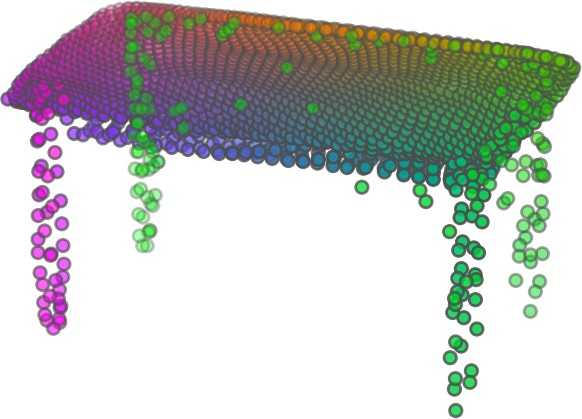}&
		\includegraphics[height=0.064000\textwidth]{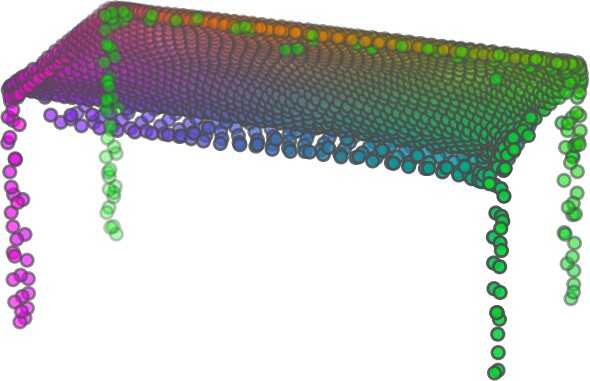}&
		\includegraphics[height=0.064000\textwidth]{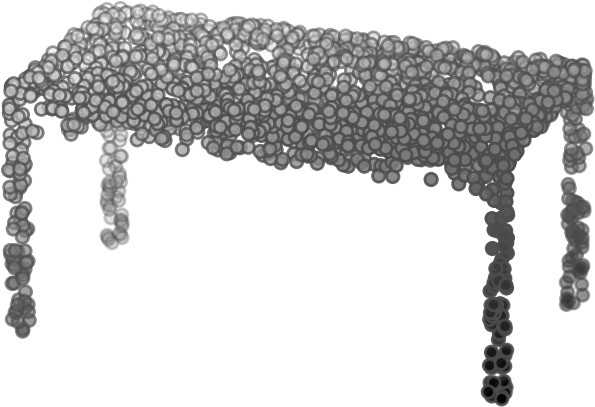}\\
		%\hline
		\includegraphics[height=0.080000\textwidth]{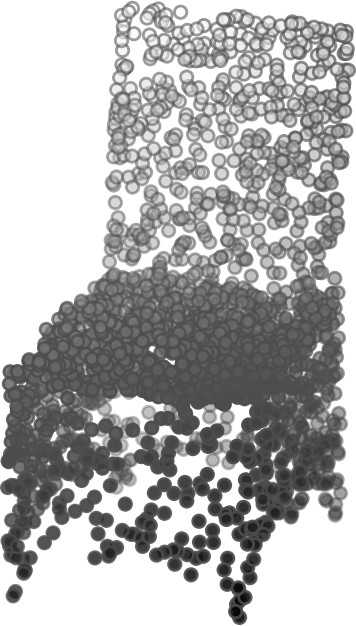}&
		\includegraphics[height=0.080000\textwidth]{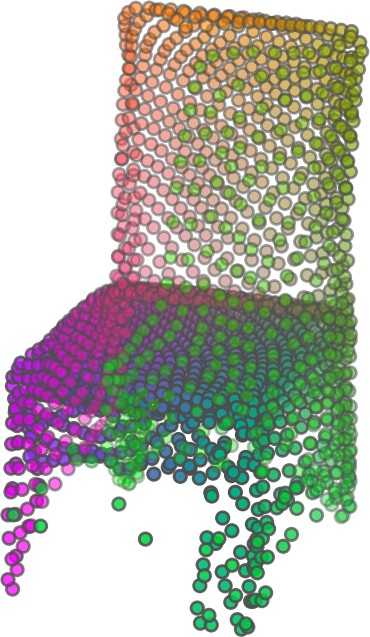}&
		\includegraphics[height=0.080000\textwidth]{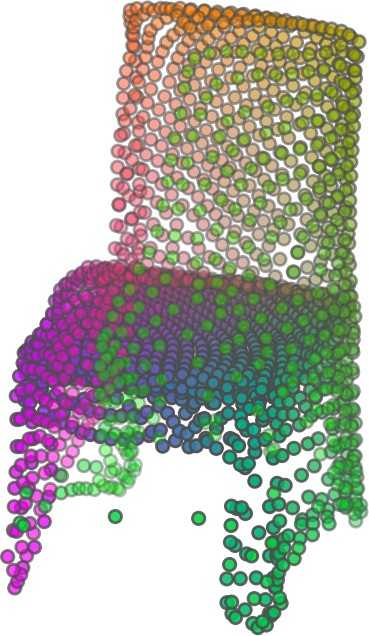}&
		\includegraphics[height=0.080000\textwidth]{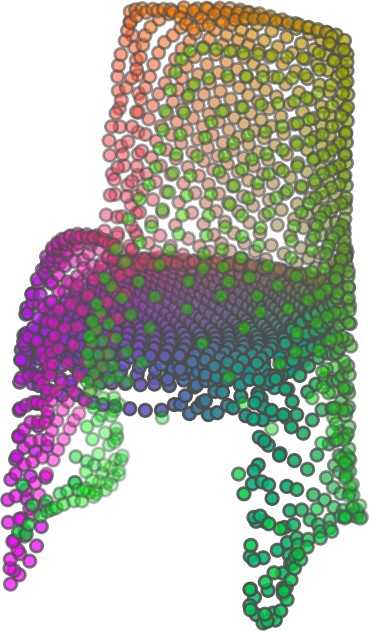}&
		\includegraphics[height=0.080000\textwidth]{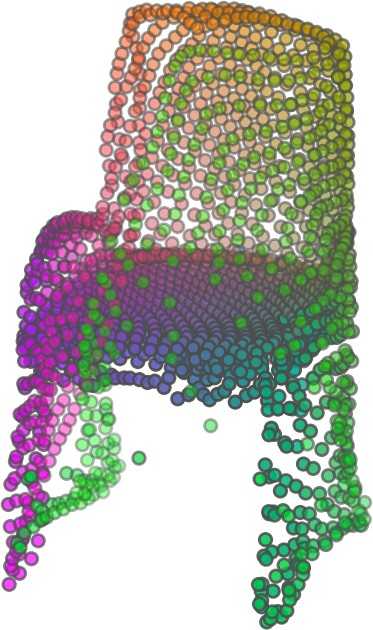}&
		\includegraphics[height=0.080000\textwidth]{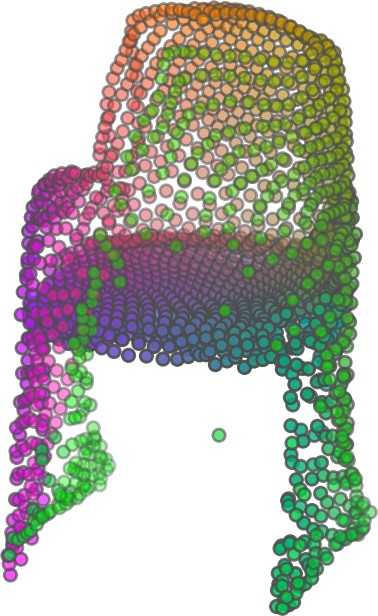}&
		\includegraphics[height=0.080000\textwidth]{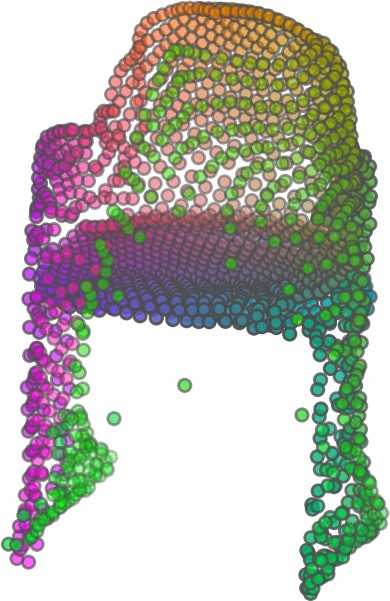}&
		\includegraphics[height=0.080000\textwidth]{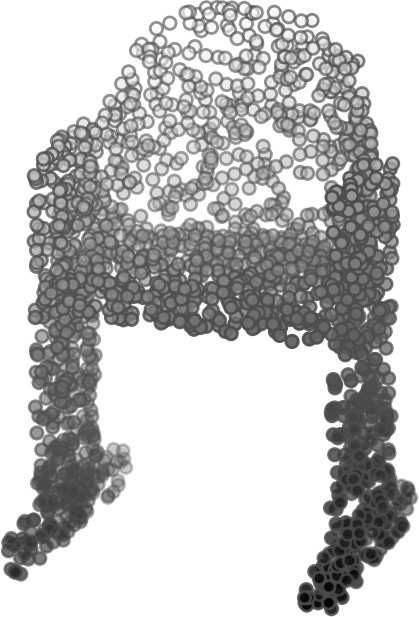}\\
		%\hline
		\includegraphics[height=0.07000\textwidth]{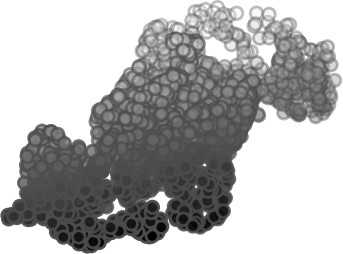}&
		\includegraphics[height=0.07000\textwidth]{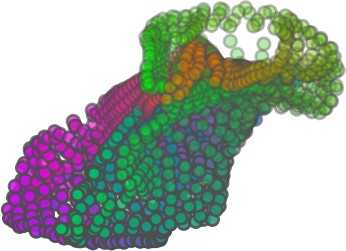}&
		\includegraphics[height=0.07000\textwidth]{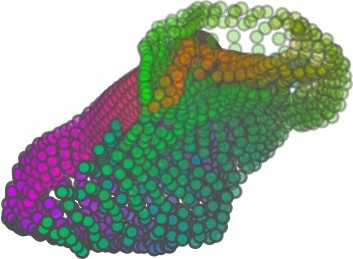}&
		\includegraphics[height=0.07000\textwidth]{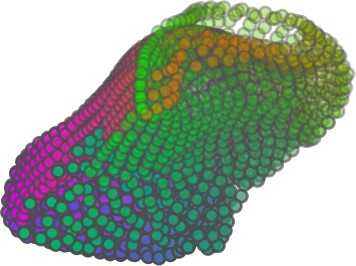}&
		\includegraphics[height=0.07000\textwidth]{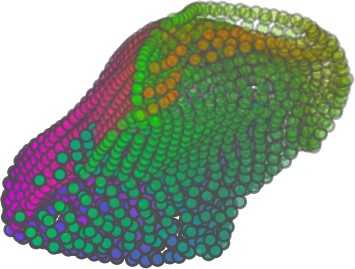}&
		\includegraphics[height=0.07000\textwidth]{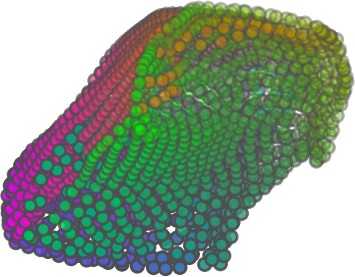}&
		\includegraphics[height=0.07000\textwidth]{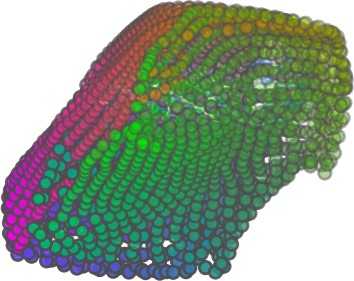}&
		\includegraphics[height=0.07000\textwidth]{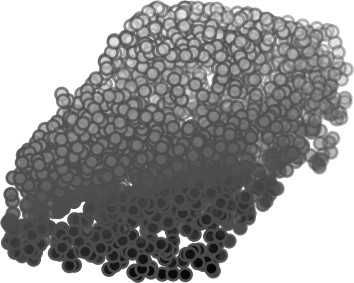}\\
		%\hline
		\hline\\[-2.5ex]
		\includegraphics[height=0.060000\textwidth]{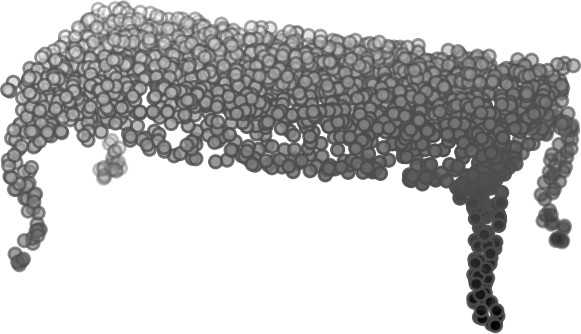}&
		\includegraphics[height=0.060000\textwidth]{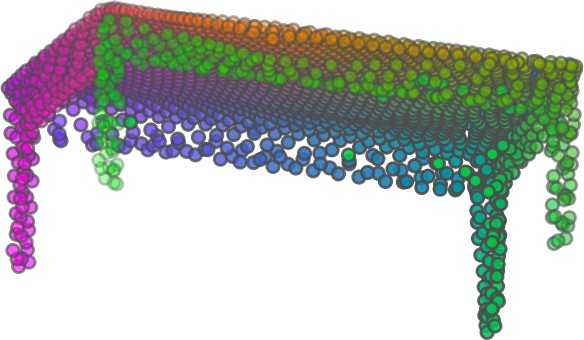}&
		\includegraphics[height=0.0650000\textwidth]{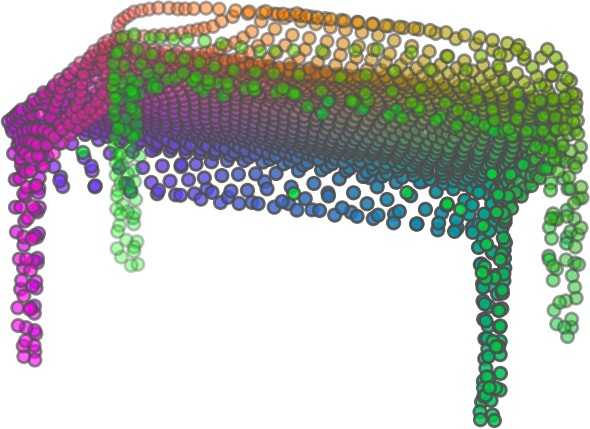}&
		\includegraphics[height=0.070000\textwidth]{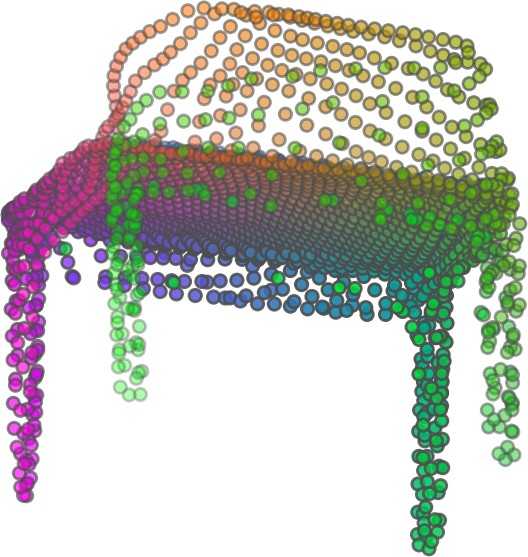}&
		\includegraphics[height=0.070000\textwidth]{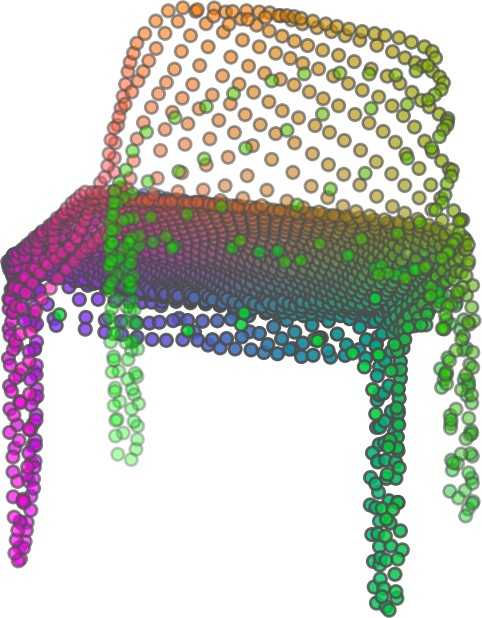}&
		\includegraphics[height=0.070000\textwidth]{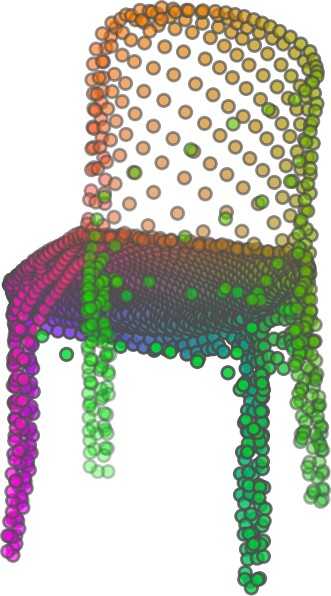}&
		\includegraphics[height=0.070000\textwidth]{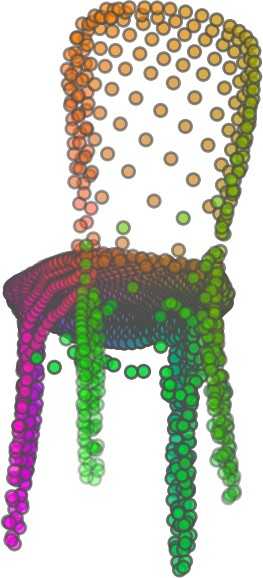}&
		\includegraphics[height=0.070000\textwidth]{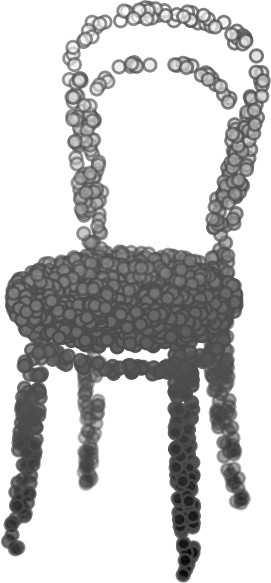}\\
		%\hline
		\includegraphics[height=0.060000\textwidth]{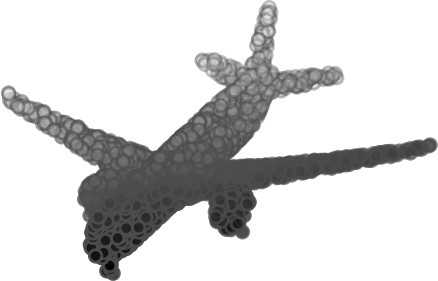}&
		\includegraphics[height=0.060000\textwidth]{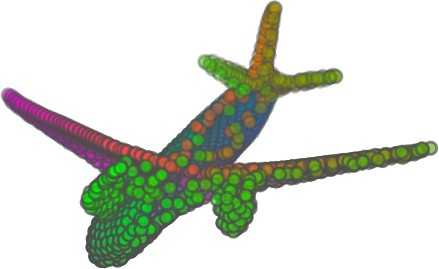}&
		\includegraphics[height=0.060000\textwidth]{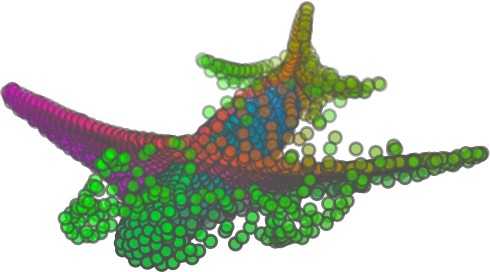}&
		\includegraphics[height=0.060000\textwidth]{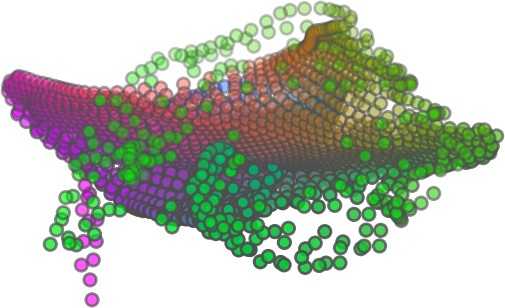}&
		\includegraphics[height=0.060000\textwidth]{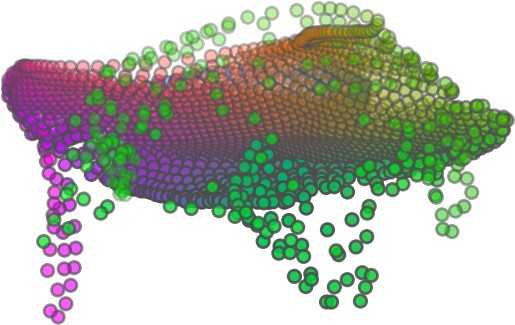}&
		\includegraphics[height=0.060000\textwidth]{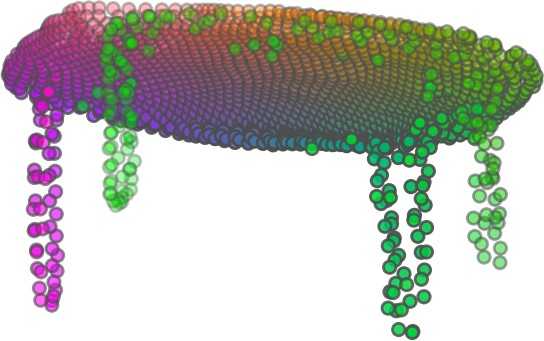}&
		\includegraphics[height=0.060000\textwidth]{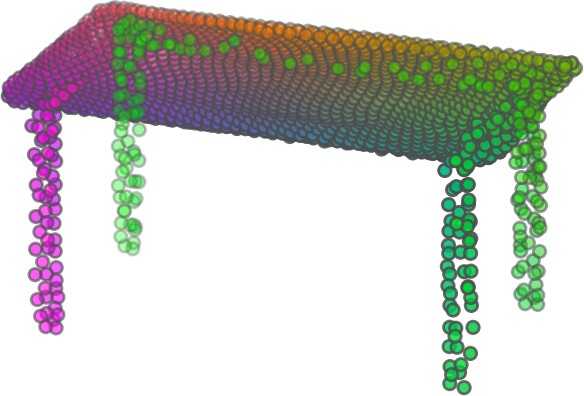}&
		\includegraphics[height=0.060000\textwidth]{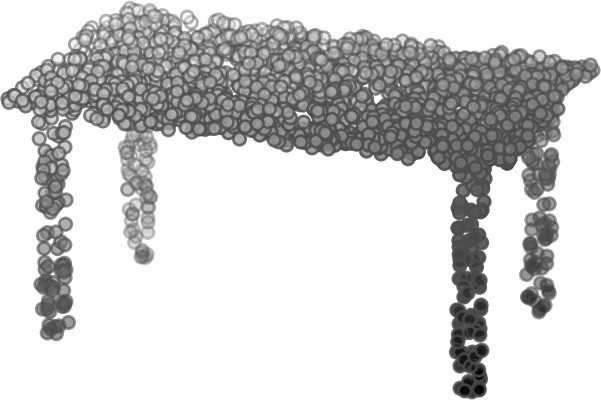}\\
		%\hline
		\includegraphics[height=0.080000\textwidth]{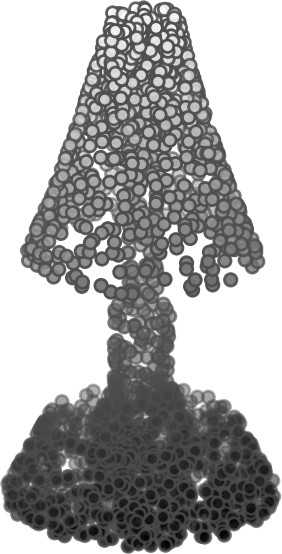}&
		\includegraphics[height=0.080000\textwidth]{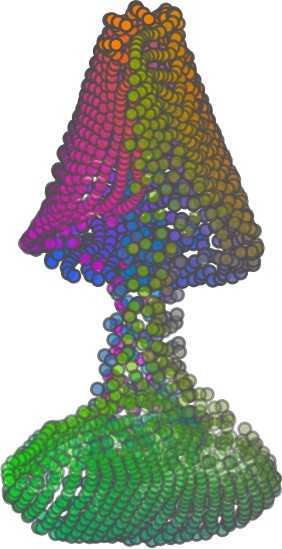}&
		\includegraphics[height=0.080000\textwidth]{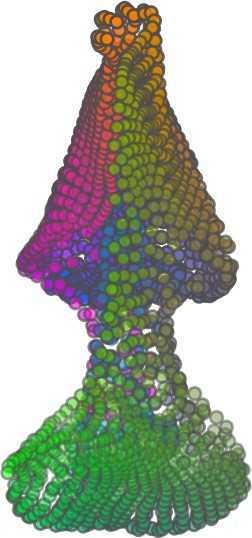}&
		\includegraphics[height=0.080000\textwidth]{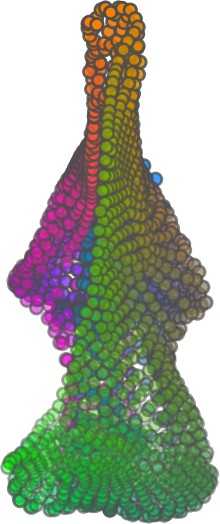}&
		\includegraphics[height=0.080000\textwidth]{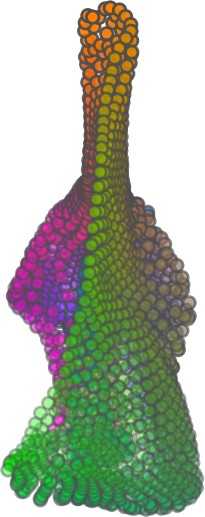}&
		\includegraphics[height=0.080000\textwidth]{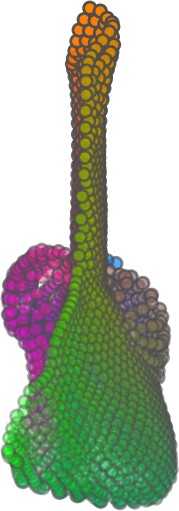}&
		\includegraphics[height=0.080000\textwidth]{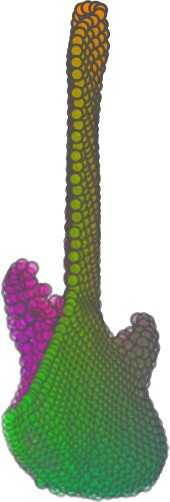}&
		\includegraphics[height=0.080000\textwidth]{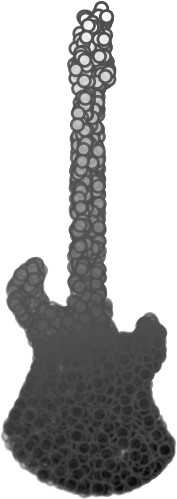}\\
		\hline
	\end{tabular}\vspace{1mm}
	\caption{Illustration of point cloud interpolation. The first 3 rows: intra-class interpolations. The last 3 rows: inter-class interpolations.\label{table:model_interpolation}}
	%\vspace{-4mm}
\end{table*}

A common method to demonstrate that the codewords have extracted the natural representations of the input is to see if the auto-encoder enables meaningful novel interpolations between two inputs in the dataset. In Table \ref{table:model_interpolation}, we show both inter-class and intra-class interpolations.
Note that we used a single AE for all shape categories for this task.

\subsection{Illustration of Point Cloud Clustering}

We also provide an illustration of clustering 3D point clouds using the codewords obtained from FoldingNet. We used the ShapeNet dataset to train the AE and obtain codewords for the ModelNet10 dataset, which we will explain in details in Section \ref{exp:transfer}. Then, we used T-SNE \cite{maaten2008visualizing} to obtain an embedding of the high-dimensional codewords in $\mathbb{R}^2$. The parameter ``perplexity'' in T-SNE was set as 50. We show the embedding result in Figure~\ref{fig:clustering}. From the figure, we see that most classes are easily separable except \{dresser (violet) v.s. nightstand (pink)\} and \{desk (red) v.s. table (yellow)\}. We have visually checked these two pairs of classes, and found that many pairs cannot be easily distinguished even by a human. In Table~\ref{table:confusion_set}, we list the most common mistakes made in classifying the ModelNet10 dataset.

\begin{figure}
	\centering
	\includegraphics[scale=0.34]{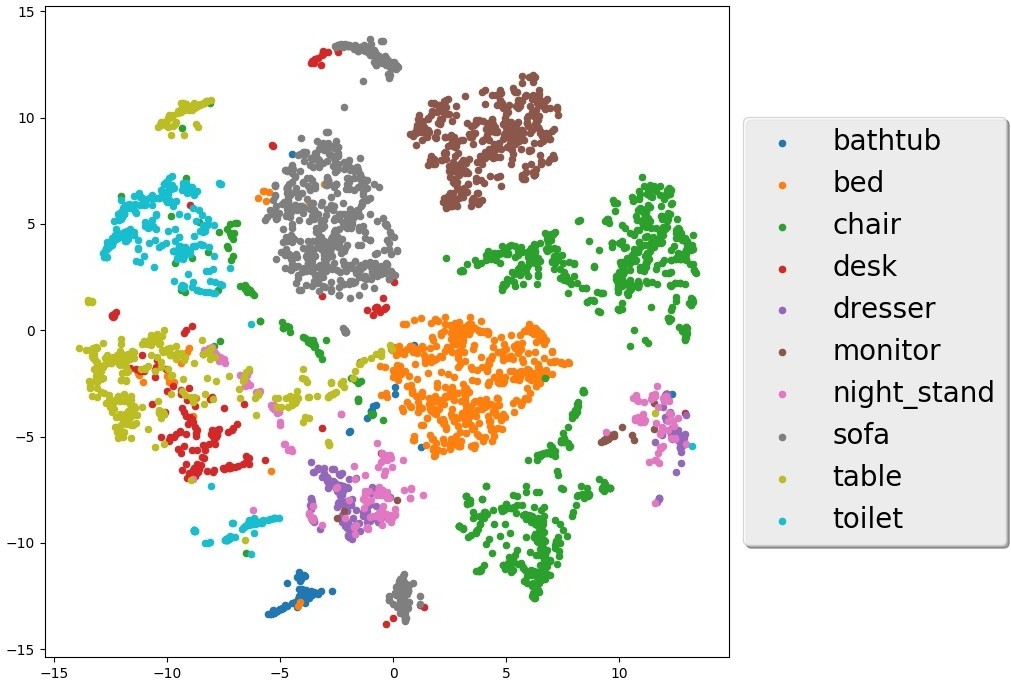}
	\caption{The T-SNE clustering visualization of the codewords obtained from FoldingNet auto-encoder. \label{fig:clustering}} \vspace{-1mm}
\end{figure}

\begin{table}[htb]
	%	\small
	\centering
	\begin{tabular}{c c c}
		\hline
		Item 1 & Item 2 & Number of mistakes\\
		\hline
		\hline
		dresser  & night stand   & 19 \\
		table    & desk          & 15 \\
		bed      & bath tub      & 3 \\
		night stand & table      & 3 \\
		\hline
	\end{tabular}\\[1ex]
	\caption{The first four types of mistakes made in the classification of ModelNet10 dataset. Their images are shown in the Supplementary Section~\ref{sec:sup:svm-modelnet10-details}. \label{table:confusion_set}} \vspace{-3mm}
\end{table}

\subsection{Transfer Classification Accuracy}\label{exp:transfer}

In this section, we show the efficiency of FoldingNet in representation learning and feature extraction from 3D point clouds. In particular, we follow the routine from \cite{wu2016learning,achlioptas2017representation} to train a linear SVM classifier on the ModelNet dataset \cite{wu20153d} using the codewords (latent representations) obtained from the auto-encoder, while training the auto-encoder from the ShapeNet dataset \cite{chang2015shapenet}. The train/test splits of the ModelNet dataset in our experiment is the same as in \cite{qi2016pointnet,wu2016learning}. The point-cloud-format of the ShapeNet dataset is obtained by sampling random points on the triangles from the mesh models in the dataset. It contains 57447 models from 55 categories of man-made objects. The ModelNet datasets are the same one used in \cite{qi2016pointnet}, and the MN40/MN10 datasets respectively contain 9843/3991 models for training and 2468/909 models for testing. Each point cloud in the selected datasets contains 2048 points with ($x$,$y$,$z$) positions normalized into a unit sphere as in \cite{qi2016pointnet}.

The codewords obtained from the FoldingNet auto-encoder is of length 512, which is the same as in \cite{achlioptas2017representation} and smaller than 7168 in \cite{wu20153d}. When training the auto-encoder, we used ADAM with an initial learning rate of 0.0001 and batch size of 1. We trained the auto-encoder for $1.6\times10^7$ iterations (i.e., 278 epochs) on the ShapeNet dataset. Similar to \cite{achlioptas2017representation,qi2016pointnet}, when training the AE, we applied random rotations to each point cloud. Unlike the random rotations in \cite{achlioptas2017representation,qi2016pointnet}, we applied the rotation that is one of the 24 axis-aligned rotations in the right-handed system. When training the linear SVM from the codewords obtained by the AE, we did not apply random rotations. We report our results in Table \ref{table:unsup_classification}. The results of \cite{kazhdan2003rotation,chen2003visual,girdhar2016learning,sharma2016vconv} are according to the report in \cite{wu2016learning,achlioptas2017representation}. Since the training of the AE and the training of the SVM are based on different datasets, the experiment shows the transfer robustness of the FoldingNet. We also include a figure (see Figure~\ref{fig:transfer40}) to show how the reconstruction loss decreases and the linear SVM classification accuracy increases during training. From Table \ref{table:unsup_classification}, we can see that FoldingNet outperforms all other methods on the MN40 dataset. On the MN10 dataset, the auto-encoder proposed in \cite{achlioptas2017representation} performs slightly better. However, the point-cloud format of the ModelNet10 dataset used in \cite{achlioptas2017representation} is not public, so the point-cloud sampling protocol of ours may be different from the one in \cite{achlioptas2017representation}. So it is inconclusive whether \cite{achlioptas2017representation} is better than ours on MN10 dataset.

\begin{figure}[t]
	\centering
	\includegraphics[scale=0.4]{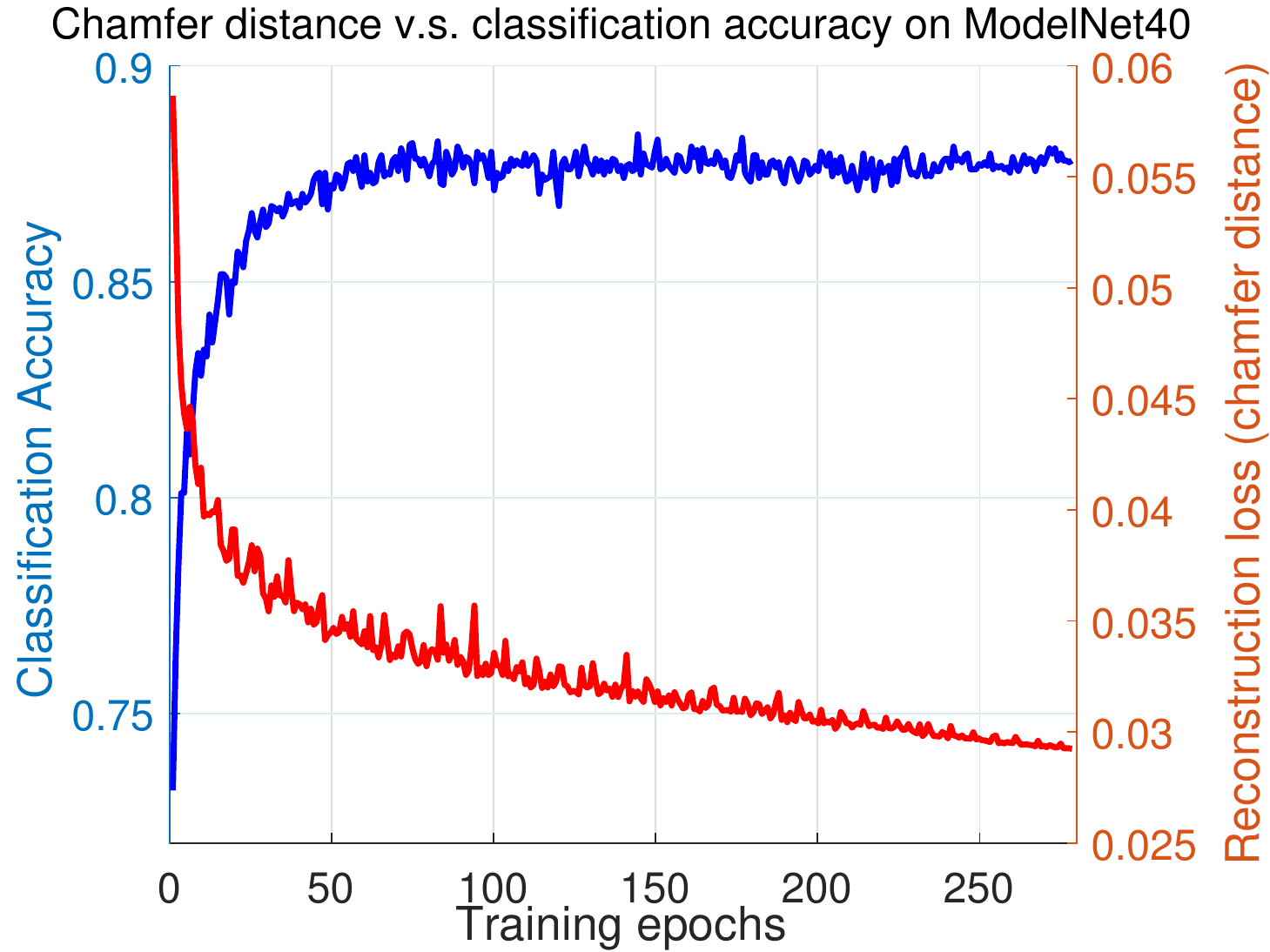}
	\caption{Linear SVM classification accuracy v.s. reconstruction loss on ModelNet40 dataset. The auto-encoder is trained using data from the ShapeNet dataset.\label{fig:transfer40}}\vspace{-2mm}
\end{figure}

\begin{table}[t]
%	\small
    \centering
    \begin{tabular}{c c c}
        \hline
        Method & MN40 & MN10\\
        \hline
        \hline
        SPH \cite{kazhdan2003rotation}         & 68.2\%       & 79.8\% \\
        LFD \cite{chen2003visual}            & 75.5\%       & 79.9\% \\
        T-L Network \cite{girdhar2016learning} & 74.4\%       & - \\
        VConv-DAE \cite{sharma2016vconv}    & 75.5\%       & 80.5\% \\
        3D-GAN \cite{wu2016learning}           & 83.3\%       & 91.0\% \\
        Latent-GAN \cite{achlioptas2017representation} & 85.7\% & {\bf95.3}\% \\
        FoldingNet (ours)         &  {\bf88.4}\%       & 94.4\% \\
        \hline
    \end{tabular}\vspace{1mm}
    \caption{The comparison on classification accuracy between FoldingNet and other unsupervised methods. All the methods train a linear SVM on the high-dimensional representations obtained from unsupervised training.\label{table:unsup_classification}}
    \vspace{-3mm}
\end{table}

\subsection{Semi-supervised Learning: What Happens when Labeled Data are Rare}

One of the main motivations to study unsupervised classification problems is that the number of labeled data is usually much smaller compared to the number of unlabeled data. In Section~\ref{exp:transfer}, the experiment is very close to this setting: the number of data in the ShapeNet dataset is large, which is more than $5.74\times 10^4$, while the number of data in the labeled ModelNet dataset is small, which is around $1.23\times 10^4$. Since obtaining human-labeled data is usually hard, we would like to test how the performance of FoldingNet degrades when the number of labeled data is small. We still used the ShapeNet dataset to train the FoldingNet auto-encoder. Then, we trained the linear SVM using only $a$\% of the overall training data in the ModelNet dataset, where $a$ can be 1, 2, 5, 7.5, 10, 15, and 20. The test data for the linear SVM are always all the data in the test data partition of the ModelNet dataset. If the codewords obtained by the auto-encoder are already linearly separable, the required number of labeled data to train a linear SVM should be small. To demonstrate this intuitive statement, we report the experiment results in Figure~\ref{fig:semi}. We can see that even if only 1\% of the labeled training data are available (98 labeled training data, which is about 1$\sim$3 labeled data per class), the test accuracy is still more than 55\%. When 20\% of the training data are available, the test classification accuracy is already close to 85\%, higher than most methods listed in Table \ref{table:unsup_classification}.

\begin{figure}[t]
	\centering
	\includegraphics[scale=0.5]{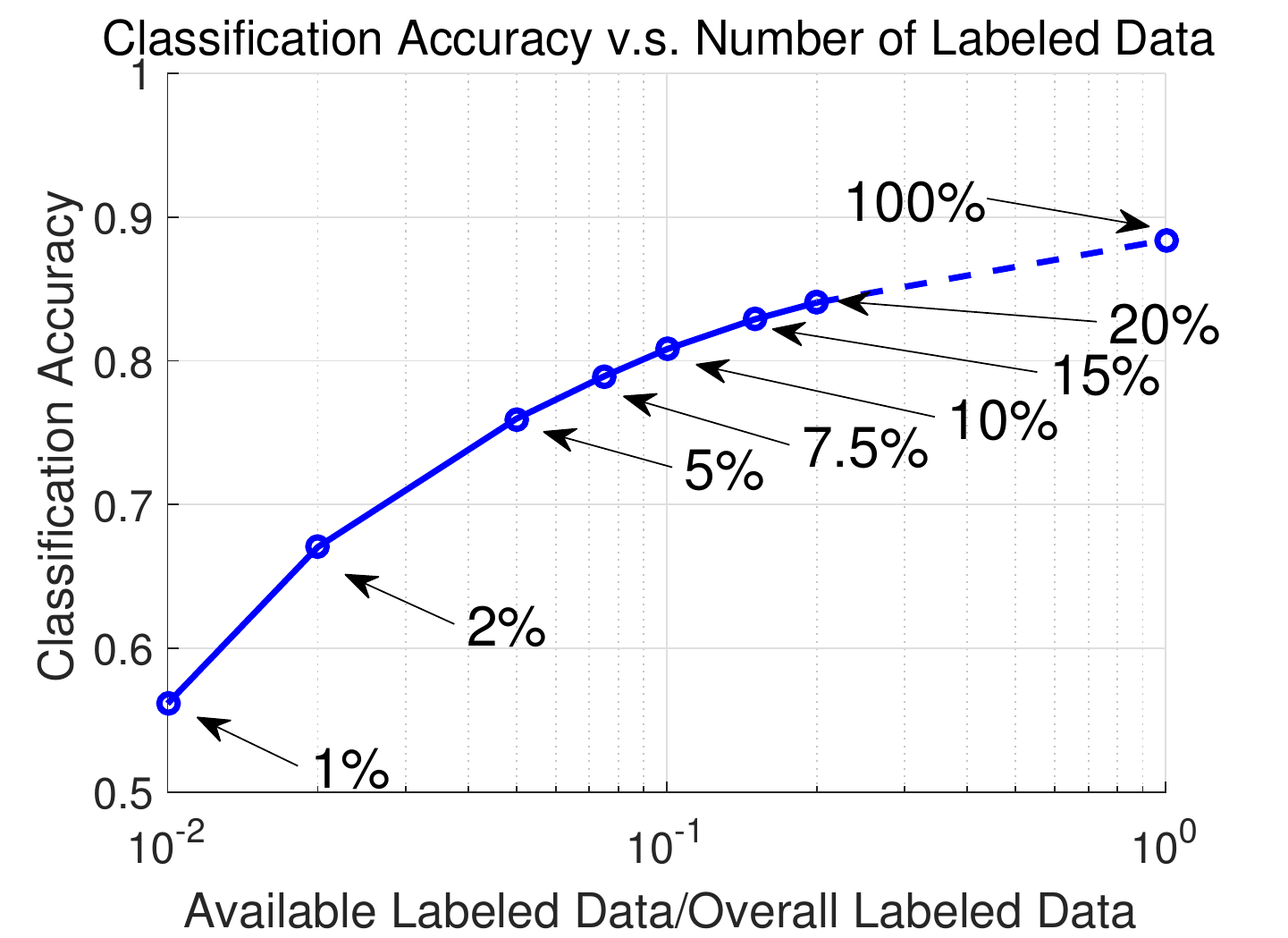}\\
	\caption{Linear SVM classification accuracy v.s. percentage of available labeled training data in ModelNet40 dataset. \label{fig:semi}} \vspace{-2mm}
\end{figure}

\subsection{Effectiveness of the Folding-Based Decoder}\label{exp:compare_FC}

\begin{figure}[t]
	\centering
	\includegraphics[scale=0.45]{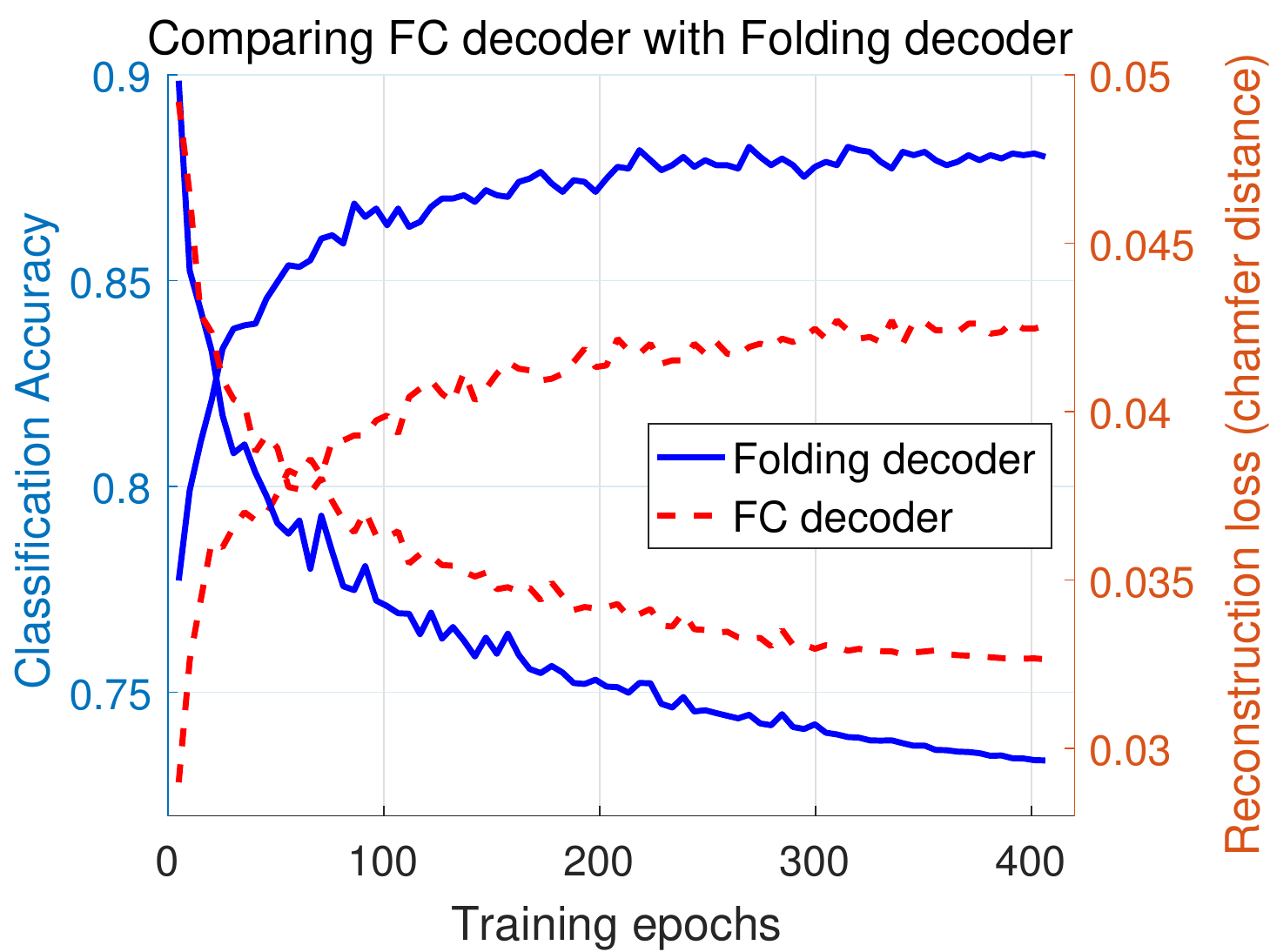}\\
	\caption{Comparison between the fully-connected (FC) decoder in \cite{achlioptas2017representation} and the folding decoder on ModelNet40. \label{fig:compare_FC}} \vspace{-2mm}
\end{figure}

In this section, we show that the folding-based decoder performs better in extracting features than the fully-connected decoder proposed in \cite{achlioptas2017representation} in terms of classification accuracy and reconstruction loss. We used the ModelNet40 dataset to train two deep auto-encoders. The first auto-encoder uses the folding-based decoder that has the same structure as in Section~\ref{sec:dec}, and the second auto-encoder uses a fully-connected three-layer perceptron as proposed in \cite{achlioptas2017representation}. For the fully-connected decoder, the number of inputs and number of outputs in the three layers are respectively \{512,1024\}, \{1024,2048\}, \{2048,2048$\times$3\}, which are the same as in \cite{achlioptas2017representation}. The output is a 2048-by-3 matrix that contains the three-dimensional points in the output point cloud. The encoders of the two auto-encoders are both the graph-based encoder mentioned in Section~\ref{sec:enc}. When training the AE, we used ADAM with an initial learning rate 0.0001, a batch size 1, for $4\times 10^6$ iterations (i.e., 406 epochs) on the ModelNet40 training dataset.

After training, we used the encoder to process all data in the ModelNet40 dataset to obtain a codeword for each point cloud. Then, similar to Section~\ref{exp:transfer}, we trained a linear SVM using these codewords and report the classification accuracy to see if the codewords are already linearly separable after encoding. The results are shown in Figure~\ref{fig:compare_FC}. During the training process, the reconstruction loss (measured in Chamfer distance) keeps decreasing, which means the reconstructed point cloud is more and more similar to the input point cloud. At the same time, the classification accuracy of the linear SVM trained on the codewords is increasing, which means the codeword representation becomes more linearly separable.

From the figure, we can see that the folding decoder almost always has a higher accuracy and lower reconstruction loss. Compared to the fully-connected decoder that relies on the unnatural ``1D order'' of the reconstructed 3D points in 3D space, the proposed decoder relies on the folding of an inherently 2D manifold corresponding to the point cloud inside the 3D space. As we mentioned earlier, this folding operation is more natural than the fully-connected decoder. Moreover, the number of parameters in the fully-connected decoder is $1.52\times 10^7$, while the number of parameters in our folding decoder is $1.05\times 10^6$, which is about 7\% of the fully-connected decoder.

One may wonder if uniformly random sampled 2D points on a plane can perform better than the 2D grid points in reconstructing point clouds. From our experiments, 2D grid points indeed provide reduced reconstruction loss than random points (Table~\ref{table:different_decoder} in Supplementary Section~\ref{sec:sup:decoders}). Notice that our graph-based max-pooling encoder can be viewed as a generalized version of the max-pooling neural network PointNet~\cite{qi2016pointnet}. The main difference is that the pooling operation in our encoder is done in a local neighborhood instead of globally (see Section~\ref{sec:enc}). In Supplementary Section~\ref{exp:comare_pointnet}, we show that the graph-based encoder architecture is better than an encoder architecture without the graph-pooling layers mentioned in Section~\ref{sec:enc} in terms of robustness towards random disturbance in point positions.

\section{Acknowledgment}
This work is supported by MERL. The authors would like to thank the helpful comments and suggestions from the anonymous reviewers, Teng-Yok Lee, Ziming Zhang, Zhiding Yu, Siheng Chen, Yuichi Taguchi, Mike Jones and Alan Sullivan.
%-------------------------------------------------

{\small
\bibliographystyle{ieee}
\bibliography{egbib}
}

\pagebreak

\section{Supplementary: {Proof of Theorem~\ref{thm:perm_inv}}}\label{sec:permutation}
Denote the input $n$-by-12 matrix by $\mathbf{L}$. Denote by $\boldsymbol{\theta}$ the codeword obtained by the encoder. Now we prove if the input is $\mathbf{PL}$ where $\mathbf{P}$ is an $n$-by-$n$ permutation matrix, the codeword obtained from the encoder is still $\boldsymbol{\theta}$.

The first part of the encoder is a per-point function, i.e., the 3-layer perceptron is applied to each row of the input matrix $\mathbf{L}$. Denote the function by $f_1$. Then, it is obvious that $f_1(\mathbf{PL})=\mathbf{P}f_1(\mathbf{L})$. The second part computes \eqref{eqn:graph_layer}. Now we prove that for \eqref{eqn:graph_layer},
\begin{equation}\label{eqn:perm0}
\mathbf{PY}=\mathbf{A}_\text{max}(\mathbf{PX})\mathbf{K}.
\end{equation}
Since $\mathbf{Y}=\mathbf{A}_\text{max}(\mathbf{X})\mathbf{K}$, we only need to prove
\begin{equation}
\mathbf{A}_\text{max}(\mathbf{PX})=\mathbf{P}\mathbf{A}_\text{max}(\mathbf{X}).
\end{equation}
Suppose the permutation operation $\mathbf{P}$ makes the $i$-th row of $\mathbf{PX}$ equal to $\mathbf{x}_{\pi(i)}$, where $\pi(\cdot)$ is a permutation function on the set of row indexes $\{1,2,\ldots,n\}$. Then, from \eqref{eqn:local_max}, the ($i$,$j$)-th entry of the matrix $\mathbf{A}_\text{max}(\mathbf{PX})$ is
\begin{equation}\label{eqn:perm1}
(\mathbf{A}_\text{max}(\mathbf{PX}))_{ij} = \text{ReLU}(\max_{k\in\mathcal{N}(\pi(i))}x_{kj}).
\end{equation}
In the meantime, the ($\pi(i)$,$j$)-th entry of $\mathbf{A}_\text{max}(\mathbf{PX})$ is
\begin{equation}\label{eqn:perm2}
(\mathbf{A}_\text{max}(\mathbf{X}))_{\pi(i)j}=\text{ReLU}(\max_{k\in\mathcal{N}(\pi(i))}x_{kj}).
\end{equation}
Since the right hand side of \eqref{eqn:perm1} and \eqref{eqn:perm2} are the same, we know that the matrix $\mathbf{A}_\text{max}(\mathbf{PX})$ can be obtained by changing the $i$-th row of $\mathbf{A}_\text{max}(\mathbf{X})$ to the $\pi(i)$-th row, which means $\mathbf{A}_\text{max}(\mathbf{PX})=\mathbf{P}\mathbf{A}_\text{max}(\mathbf{X})$.
Thus, we have proved that for the second part of the encoder, permuting the input rows is equivalent to permuting the output rows, i.e., \eqref{eqn:perm0} holds.

Therefore, if we permute the input to the encoder, the output of the graph layers also permute. Then, we apply global max-pooling to the output of the graph layers. It is obvious that the result remains the same if the rows of the input to the global max-pooling layer (or the output of the graph layers) permute. The conclusion of Theorem~\ref{thm:perm_inv} is hence proved.

\section{Supplementary: {Proof of Theorem~\ref{thm:1}}}\label{sec:proof}

We prove the existence-based Theorem~\ref{thm:1} by explicitly constructing a 2-layer perceptron and a codeword vector $\boldsymbol{\theta}$ that satisfy the stated properties.

The codeword is simply chosen as the vectorized form of the point cloud matrix $\mathbf{S}$. In particular, For a matrix $\mathbf{S}$ of size $m$-by-3, if $\mathbf{S}=[s_{jk}],j=1,2,\ldots m$ and $k=1,2,3$, the codeword vector $\boldsymbol{\theta}$ is chosen to be $\boldsymbol{\theta}=[s_{11},s_{12},s_{13},s_{21},s_{22},s_{23},\ldots,s_{m1},s_{m2},s_{m3}]$. Then, the $i$-th row after concatenation is $\mathbf{v}_i=[x_i,y_i,s_{11},s_{12},s_{13},s_{21},s_{22},s_{23},\ldots,s_{m1},s_{m2},s_{m3}]$, where $[x_i,y_i]$ is the position of the $i$-th 2D grid point. Suppose the 2D grid points have an interval $2\delta$, i.e., the distance between any two points in the 2D grid is at least $2\delta$. Further assume these $m$ grid points can all be written as $[x_i,y_i]=[(2\beta_i+1)\delta,(2\gamma_i+1)\delta]$, where $\beta_i$ and $\gamma_i$ are two integers whose absolute values are smaller than a positive constant $M$. One example of a set of 4-by-4 grid points is
\begin{equation}
\begin{matrix}
\{[-3\delta,-3\delta],&[-3\delta,-1\delta],&[-3\delta,1\delta],&[-3\delta,3\delta],\\
[-1\delta,-3\delta],&[-1\delta,-1\delta],&[-1\delta,1\delta],&[-1\delta,3\delta],\\
[1\delta,-3\delta], &[1\delta,-1\delta], &[1\delta,1\delta], &[1\delta,3\delta],\\
[3\delta,-3\delta], &[3\delta,-1\delta], &[3\delta,1\delta], &[3\delta,3\delta]\}.
\end{matrix}
\end{equation}
In this case, the choice of $M$ is 4. Also assume that the output point cloud is bounded inside 3-dimensional box of length 2 centered at the origin, i.e., $|s_{ij}|\le 1$.

Now, we construct a 2-layer perceptron $f$ that takes the rows $\mathbf{v}_i$ as inputs and provides the outputs $f(\mathbf{v}_i)=[s_{i1},s_{i2},s_{i3}]$, for $i=1,2,\ldots,m$. The input layer takes the vector intput $\mathbf{v}_i$ which has $3m+2$ scalars. The hidden layer has $3m$ neurons. The output layer provides three scalar outputs $[s_{i1},s_{i2},s_{i3}]$. The $3m$ neurons in the hidden layer are partitioned into $m$ groups of 3 neurons. The $k$-th neuron ($k=1,2,3$) in the $j$-th group ($j=1,2,3,\ldots,m$) is only connected to three inputs $x_i$, $y_i$ and $[s_{j,k}]$, and it computes a linear combination of $x_i$, $y_i$ and $s_{j,k}$ with weights
\begin{equation}
\begin{split}
\alpha_{j1}= u^2 x_j,\\
\alpha_{j2} = u y_j,\\
\alpha_{j3} = 1,\\
\end{split}
\end{equation}
and bias
\begin{equation}
b=-u^2x_j^2-uy_j^2
\end{equation}
where $u$ is a positive constant to be specified later. Suppose the linear combination output is $y_{j,k}$. The linear combination is followed by a nonlinear activation function\footnote{It is not hard to prove that this function can be obtained by concatenating ReLU functions with appropriate bias terms. We specifically avoid using the ReLU function in order not to hinder the main intuition. In all of our experiments, we use ReLU activation functions.} that computes the following
\begin{equation}\label{eqn:z}
z_{j,k}=\left\{\begin{matrix}
y_{j,k}, & \text{if } |y_{j,k}|<c,\\
0, & \text{if }|y_{j,k}|\ge c,
\end{matrix}\right.
\end{equation}
where $c$ is a constant to be specified later. The outputs of the activation functions are linearly combined to produce the final output. There are three neurons in the output layer. The $k$-th neuron ($k$=1,2,3) computes
\begin{equation}\label{eqn:w}
w_k = \sum_{j=1}^m z_{j,k}.
\end{equation}
We assume the parameters $(\delta,u,c,M)$ satisfy
\begin{align}
&u>0,c>0,\delta>0,M>0,\label{eqn:ass0}\\
&u\delta^2>c+1,\label{eqn:ass1}\\
&u>8M^2+4M+1,\label{eqn:ass2}\\
&c>1.\label{eqn:ass3}
\end{align}

Now we prove that for this perceptron, the final output $[w_1,w_2,w_3]$ is indeed $[s_{i1},s_{i2},s_{i3}]$ when the input to the perceptron is $\mathbf{v}_i$. For the $i$-th input $\mathbf{v}_i=[x_i,y_i,s_{11},s_{12},s_{13},s_{21},s_{22},s_{23},\ldots,s_{m1},s_{m2},s_{m3}]$, the $k$-th neuron in the $j$-th group in the hidden layer computes the following linear combination
\begin{equation}
\begin{split}
y_{j,k} = &\alpha_{j1} x_i + \alpha_{j2} y_i + \alpha_{j3} s_{j,k}+b\\
= &u^2 x_jx_i+u y_jy_i + s_{j,k}-u^2x_j^2-uy_j^2\\
= &u^2 x_j(x_i-x_j) +uy_j(y_i-y_j) + s_{j,k}.
\end{split}
\end{equation}
Notice that we have assumed
$[x_i,y_i]=[(2\beta_i+1)\delta,(2\gamma_i+1)\delta],\forall i$. So we have
\begin{equation}
\begin{split}
&y_{j,k}= u^2 x_j(x_i-x_j) +uy_j(y_i-y_j) + s_{j,k}\\
=&2u^2\delta^2(2\beta_j+1)(\beta_i-\beta_j)+2u\delta^2(2\gamma_j+1)(\gamma_i-\gamma_j) + s_{j,k}\\
=&u^2\delta^2 m_1 + u\delta^2 m_2+s_{j,k},
\end{split}
\end{equation}
where the two integer constants $m_1=2(2\beta_j+1)(\beta_i-\beta_j)$ and $m_2=2(2\gamma_j+1)(\gamma_i-\gamma_j)$, and $m_1=0$ only if $x_i=x_j$ and $m_2=0$ only if $y_i=y_j$. Since the absolute values of $\beta_i$, $\beta_j$, $\gamma_i$ and $\gamma_j$ are all smaller than $M$, we have
\begin{equation}
|m_1|\le 2|2\beta_j+1|\cdot|\beta_i-\beta_j|<2(2M+1)\cdot 2M = 8M^2+4M.
\end{equation}
Similarly, we have
\begin{equation}
|m_2|\le 2|2\gamma_j+1|\cdot|\gamma_i-\gamma_j|<2(2M+1)\cdot 2M = 8M^2+4M.
\end{equation}
Now we consider 3 possible cases:
\begin{itemize}
\item $|m_1|\ge 1$: In this case,
\begin{equation}
\begin{split}
|y_{j,k}|=& |u^2\delta^2 m_1 + u\delta^2 m_2+s_{j,k}|\\
>&u^2\delta^2 |m_1| - u\delta^2 |m_2|-|s_{j,k}|\\
>&u^2\delta^2 -u\delta^2(8M^2+4M)-1\\
=&u\delta^2 [u-(8M^2+4M)]-1\\
\overset{(a)}{>}&(c+1)\cdot 1 -1=c,
\end{split}
\end{equation}
where step (a) follows from the assumption \eqref{eqn:ass1}.
\item $m_1=0$ but $|m_2|\ge 1$: In this case,
\begin{equation}
\begin{split}
|y_{j,k}|=&|u\delta^2m_2+s_{j,k}|\\
\ge&u\delta^2|m_2|-|s_{j,k}|\\
\ge &u\delta^2\overset{(a)}{\ge}c+1>c,
\end{split}
\end{equation}
where step (a) follows from assumption \eqref{eqn:ass2}.

\item $m_1=m_2=0$. In this case,
\begin{equation}
|y_{j,k}|=|s_{j,k}|\le 1\overset{(a)}{<}c,
\end{equation}
where step (a) follows from assumption \eqref{eqn:ass3}.
\end{itemize}
Notice that the first two cases are equivalent to $i\neq j$ and the last case is equivalent to $i=j$. Thus, from \eqref{eqn:z}, we have
\begin{equation}
z_{j,k}=
\left\{\begin{matrix}
s_{j,k}, & \text{if } j=i,\\
0, & \text{if } j\neq i.
\end{matrix}\right.
\end{equation}
Thus, from \eqref{eqn:w}, the final output is
\begin{equation}
w_k = \sum_{j=1}^m z_{j,k}=s_{i,k},k=1,2,3,
\end{equation}
which means the output is indeed $[s_{i,1},s_{i,2},s_{i,3}]$ when the input is $\mathbf{v}_i$. This concludes the proof.

\section{Supplementary: Decoder Variations \label{sec:sup:decoders}}

\begin{table}[b]
	%	\small
	\centering
	\begin{tabular}{l|c|c|c}
		\hline
		Grid Setting      & \#Folds& Test Cls. Acc. & Test Loss\\
		\hline
		regular 2D & 2      & 88.25\% & 0.0296 \\
		\hline
		regular 2D & 3      & 88.41\% & 0.0290 \\
		\hline
		regular 1D & 2      & 86.71\% & 0.0355\\
		\hline
		regular 3D & 2      & 88.41\% & 0.0284 \\
		\hline
		uniform 2D & 2      & 87.12\% & 0.0321 \\
		\hline
	\end{tabular}\vspace{1mm}
	\caption{
		Comparison between different FoldingNet decoders.
		``Uniform'': the grid is uniformly random sampled.
		``Regular'': the grid is regularly sampled with fixed spacings. }
	\label{table:different_decoder}
\end{table}

The current decoder design has two consecutive folding operations that apply on a 2D grid. Therefore, one may wonder if the performance of FoldingNet can be improved if we utilize (1) more folding operations or (2) the same number of folding operations on regular grids of different dimensions. In this section, we report the results for these different settings. The experimental settings are the same with Section~\ref{exp:compare_FC}. The experiment results are shown in Table \ref{table:different_decoder}. As one can see from line 1 and line 2, increasing the number of folding operations does not significantly increase the performance. Comparing line 1 and line 3, one can see that a 2D grid is better than a 1D grid for both classification and reconstruction. From line 1 and line 4, one can see that a 3D grid only brings a marginal improvement. As we discussed in the introduction, this is because the intrinsic dimensionality of data in the ShapeNet and ModelNet datasets is 2, as they are sampled from object surfaces. If point clouds are intrinsically volumetric, we believe using a 3D grid in the decoder is more suitable. In addition, we also tried to generate the fixed grid by uniformly random sampling in the square. However, it leads to slightly worse results. We believe it is caused by the local density variation introduced by the random sampling.

\section{Supplementary: Folding by Deconvolution}\label{sec:devonvolution}
The folding operation in Definition~\ref{def:folding} is essentially a per-point 2D-to-3D function from a 2D grid to a 3D surface. A natural question to ask is whether introducing explicit correlations in the functions imposed on neighboring grid points can help improve the performance. We noted that there is a closely related work on reconstructing 3D point sets using side information from images \cite{fan2016point}. The point reconstruction network in \cite{fan2016point} uses deconvolution to fuse information on the regular grid structure imposed by the image, which is similar to the idea above. Here, we compare a deconvolution network with FoldingNet on the reconstruction performance. The feature sizes of the deconvolution network (C$\times$H$\times$W) are 512$\times$1$\times$1$\to$256$\times$3$\times$3$\to$128$\times$5$\times$5$\to$64$\times$15$\times$15$\to$3$\times$45$\times$45 with kernel sizes 3, 3, 5, 5. The comparison is shown in Table~\ref{table:deconv}. We conjecture that deconvolution goes beyond point-wise operations, thus imposes a stronger constraint on the smoothness of the reconstructed surface. Thus, its reconstruction is worse (although with comparable classification accuracy). On the other hand, the use of grids with point-wise MLP in FoldingNet only impose an implicit constraint, thus leading to better reconstructions.
\begin{table}[htb]
%	\small
    \centering
    \begin{tabular}{c c c c}
        \hline
         & Cl. Acc. & Tst. Loss & \# Params.\\
        \hline
%        \hline
        FoldingNet & 88.41\% & 0.0296 & 1.0$\times$10$^6$\\
        Deconv & 88.86\%  & 0.0319 & 1.7$\times$10$^6$\\
        \hline
    \end{tabular} \vspace{1mm}
	\caption{Comparison of two different implementations of the folding operation.\label{table:deconv}}
\end{table}

\section{Supplementary: Robustness of the graph-based encoder}\label{exp:comare_pointnet}

Here, we use one experiment to show that the graph-pooling layers are useful in maintaining the good performance of the FoldingNet when the data is subject to random noise. The following experiment compares FoldingNet with a deep auto-encoder that has the same folding-based decoder architecture but a different encoder architecture in which the graph-based max-pooling layers are removed. The setting of the experiment is the same as in Section~\ref{exp:compare_FC} except that 5 percents of the points in each point cloud in the ModelNet40 dataset are randomly shifted to other positions (but still within the bounding box of the original point cloud). We use this noisy data to see how the performances degrade for the graph-based encoder and the encoder without graph-based max-pooling layers. The results are reported in Figure~\ref{fig:compare_pointnet}. We can see that when the graph-based max-pooling layers are removed, the performance degrades by approximately 2 percents when noise is injected into the dataset. However, the classification accuracy of FoldingNet does not change much (when compared with Figure~\ref{fig:compare_FC} in Section~\ref{exp:compare_FC}). Thus, it can be seen that the graph-based encoder can make FoldingNet more robust.

\begin{figure}
	\centering
	\includegraphics[scale=0.55]{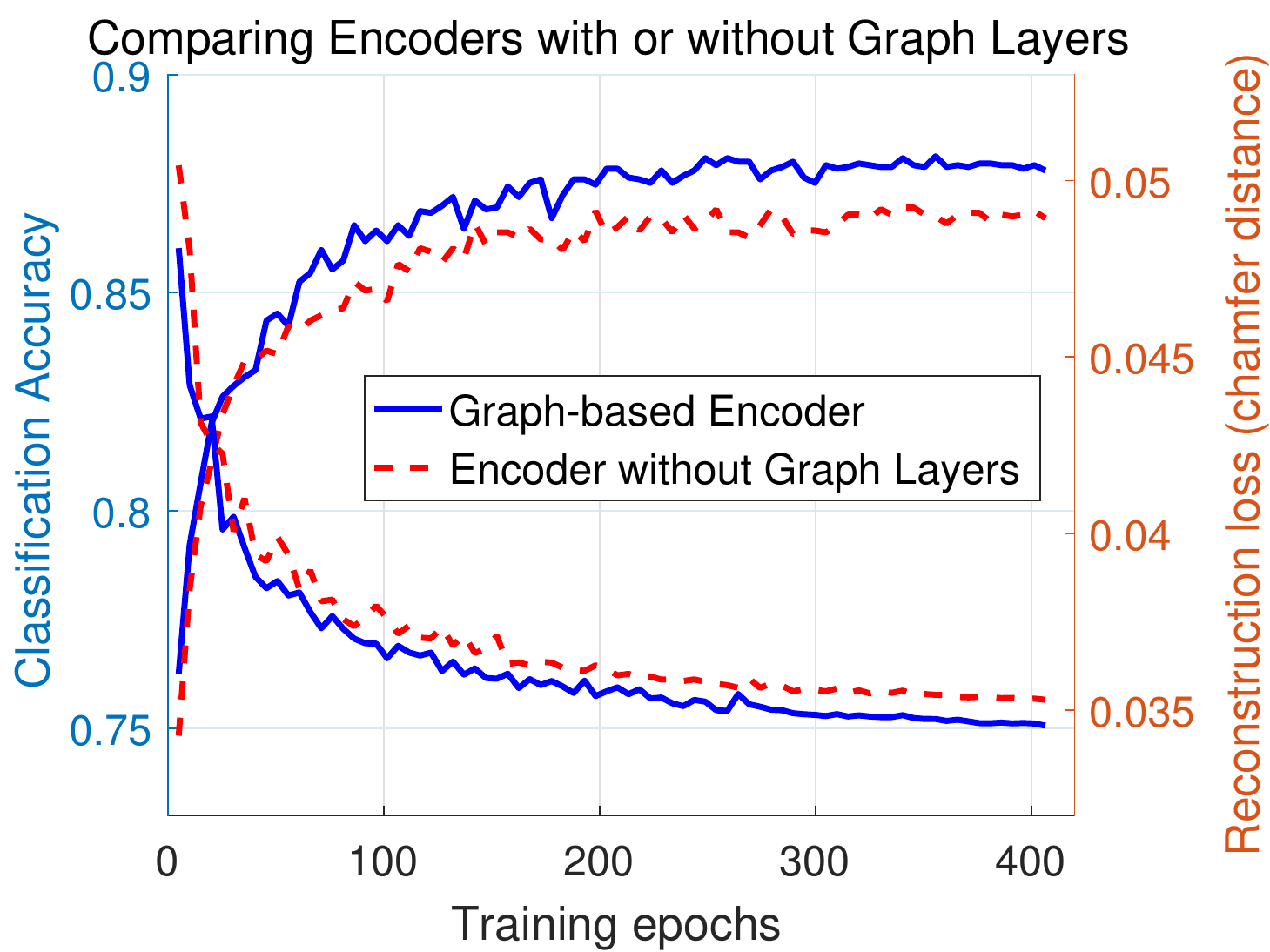}\\
	\caption{Comparison between the graph-based encoder in Section~\ref{sec:enc} and the encoder from which the graph-based max-pooling layers are removed. The encoder with no graph-based layers is similar to the one proposed in \cite{qi2016pointnet} which is for a different goal (supervised learning).}\vspace{-0mm} \label{fig:compare_pointnet}
\end{figure}

\section{Supplementary: More Details on the Linear SVM Experiment on ModelNet10 \label{sec:sup:svm-modelnet10-details}}

The classification accuracy obtained in Section~\ref{exp:transfer} on MN10 dataset is 94.4\%. We stated in Section 4.5 that many pairs which are wrongly classified are actually hard to distinguish even by a human. In the table on the next page, we list all the incorrectly classified models and their point cloud representations. A phrase like ``table $\to$ desk'' means the point cloud has label ``table'' but it is wrongly classified as ``desk'' by the linear SVM.

\begin{table*}
\begin{tabularx}{\textwidth}{XXX}
%\hline
\hline\\[-2ex]
\includegraphics[height=0.070000\textwidth]{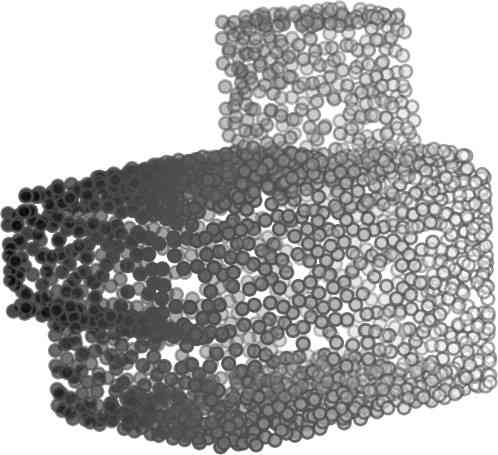}toilet $\to$ bed&
\includegraphics[height=0.070000\textwidth]{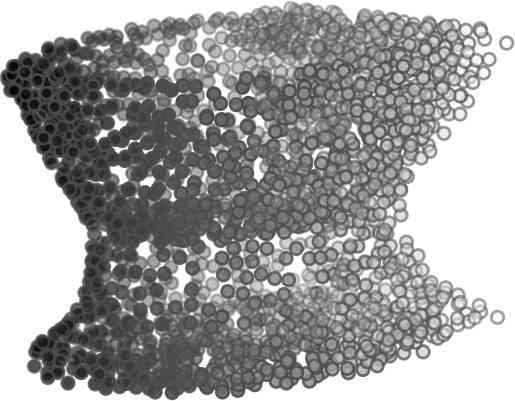}toilet $\to$ bathtub&
\includegraphics[height=0.070000\textwidth]{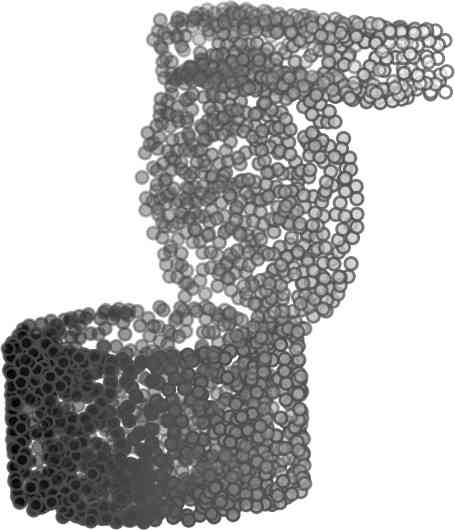}toilet $\to$ chair\\
%\hline
\includegraphics[height=0.070000\textwidth]{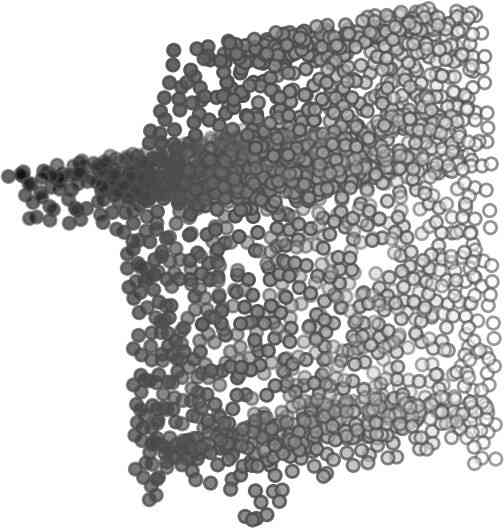}dresser $\to$ night\_stand&
\includegraphics[height=0.070000\textwidth]{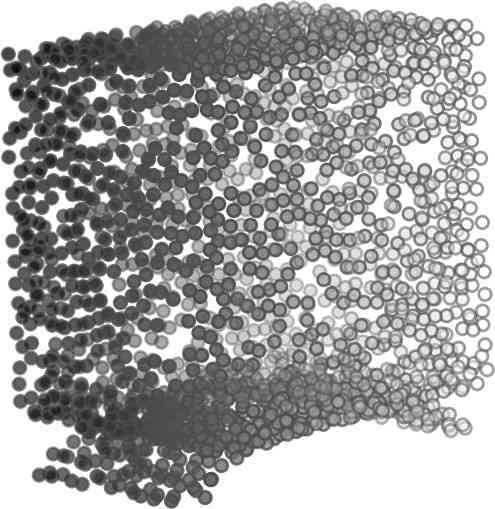}dresser $\to$ night\_stand&
\includegraphics[height=0.070000\textwidth]{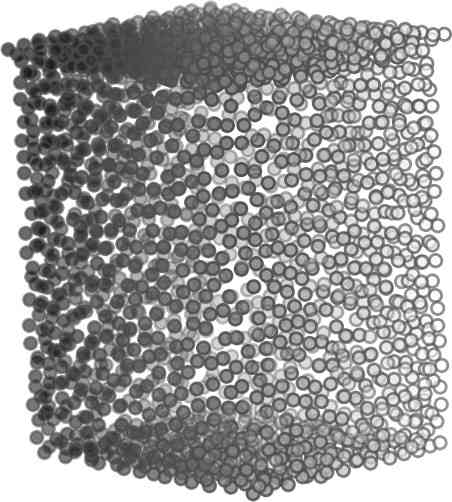}dresser $\to$ night\_stand\\
%\hline
\includegraphics[height=0.070000\textwidth]{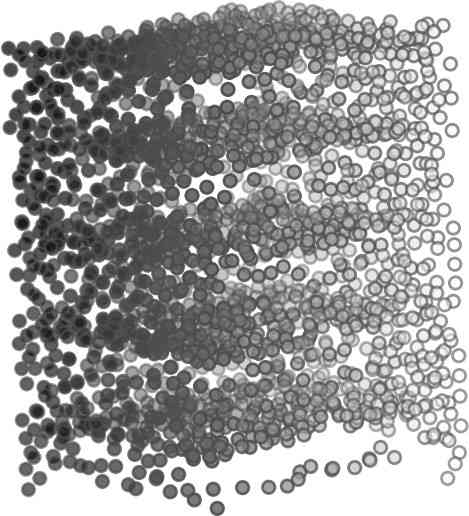}dresser $\to$ night\_stand&
\includegraphics[height=0.070000\textwidth]{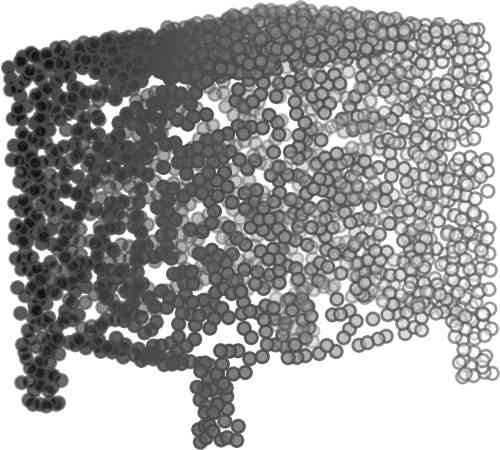}dresser $\to$ night\_stand&
\includegraphics[height=0.070000\textwidth]{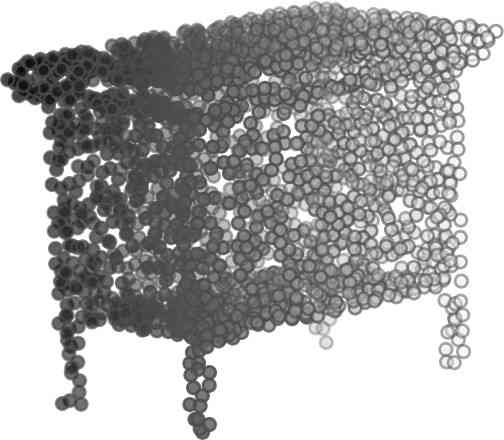}dresser $\to$ night\_stand\\
%\hline
\includegraphics[height=0.070000\textwidth]{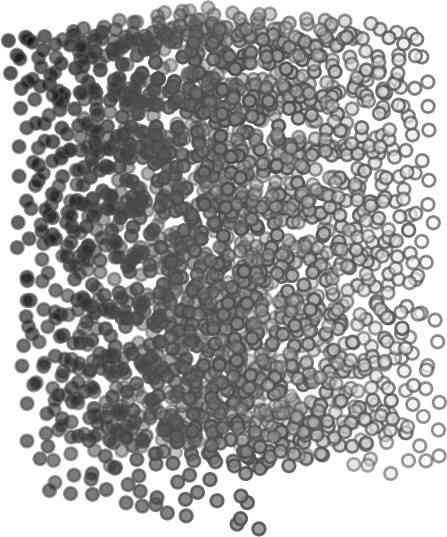}dresser $\to$ night\_stand&
\includegraphics[height=0.070000\textwidth]{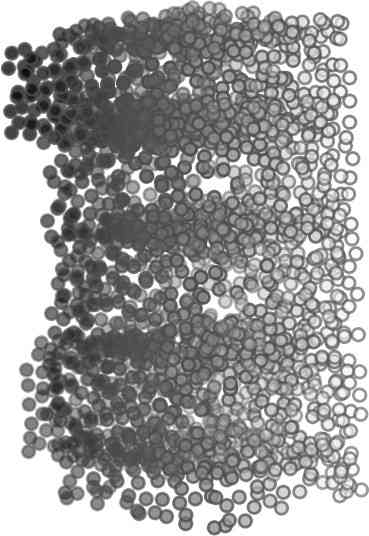}dresser $\to$ night\_stand&
\includegraphics[height=0.070000\textwidth]{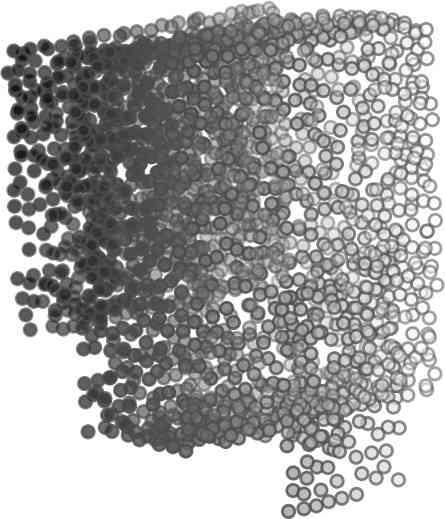}monitor $\to$ dresser\\
%\hline
\includegraphics[height=0.070000\textwidth]{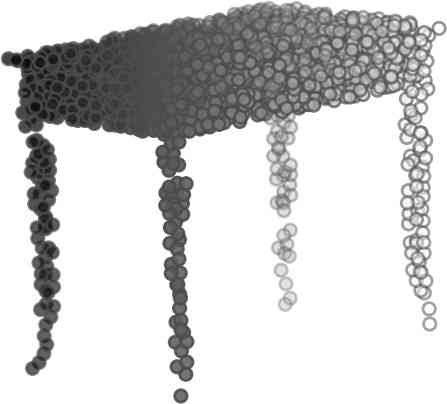}desk $\to$ table&
\includegraphics[height=0.070000\textwidth]{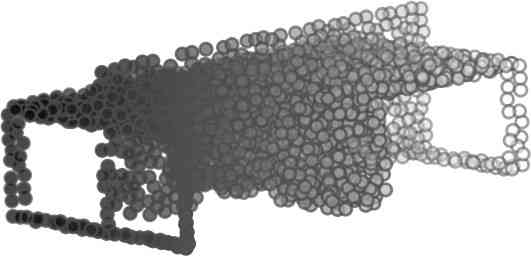}desk $\to$ table&
\includegraphics[height=0.070000\textwidth]{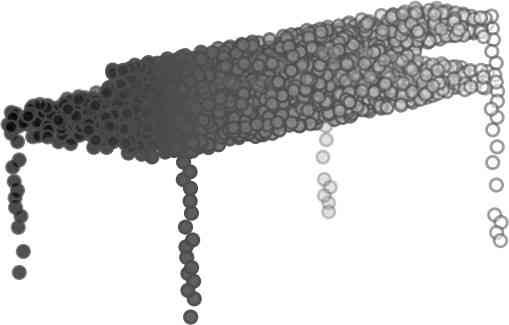}desk $\to$ sofa\\
%\hline
\includegraphics[height=0.070000\textwidth]{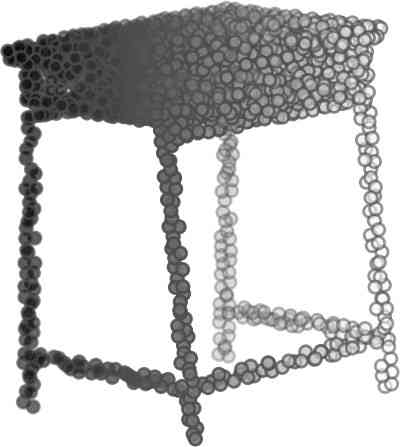}desk $\to$ night\_stand&
\includegraphics[height=0.070000\textwidth]{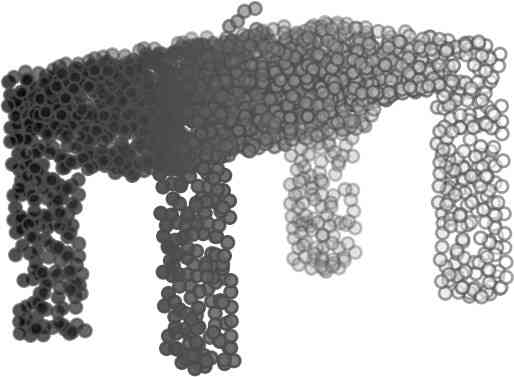}desk $\to$ table&
\includegraphics[height=0.070000\textwidth]{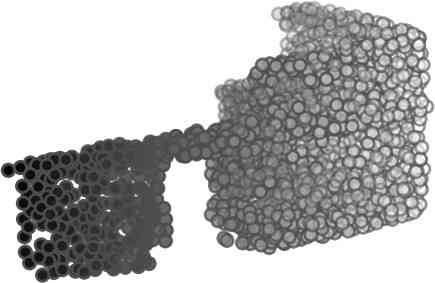}desk $\to$ sofa\\
%\hline
\includegraphics[height=0.070000\textwidth]{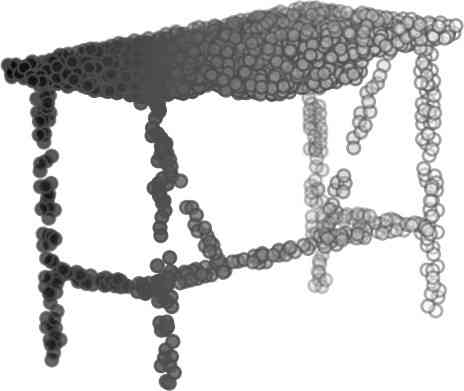}desk $\to$ table&
\includegraphics[height=0.070000\textwidth]{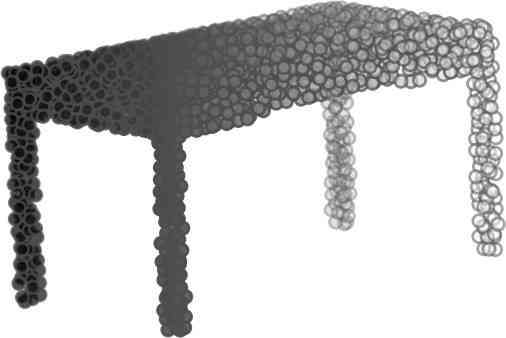}desk $\to$ table&
\includegraphics[height=0.050000\textwidth]{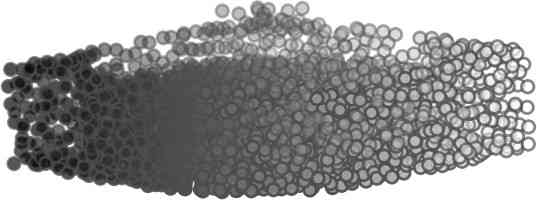}bathtub $\to$ bed\\
%\hline
\includegraphics[height=0.050000\textwidth]{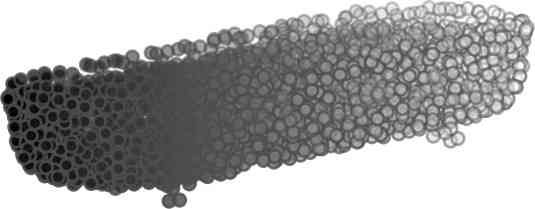}bathtub $\to$ table&
\includegraphics[height=0.050000\textwidth]{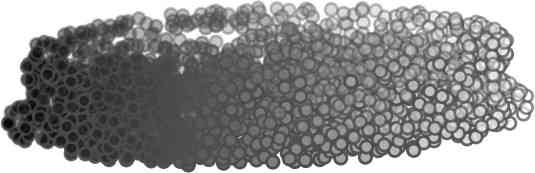}bathtub $\to$ bed&
\includegraphics[height=0.070000\textwidth]{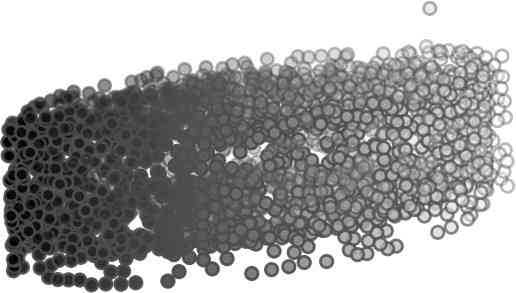}bathtub $\to$ table\\
%\hline
\includegraphics[height=0.050000\textwidth]{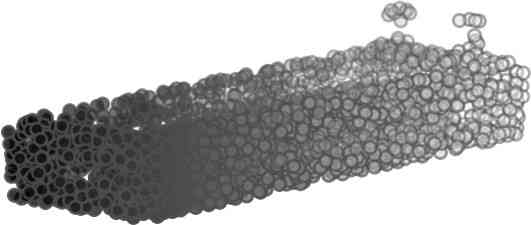}bathtub $\to$ bed&
\includegraphics[height=0.070000\textwidth]{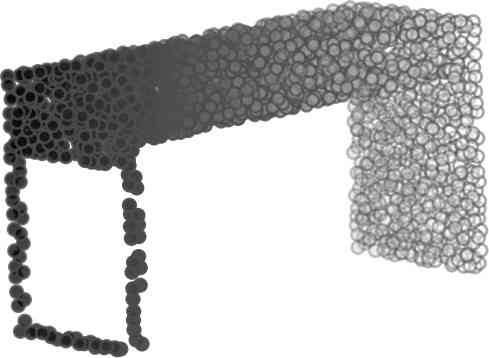}table $\to$ desk&
\includegraphics[height=0.070000\textwidth]{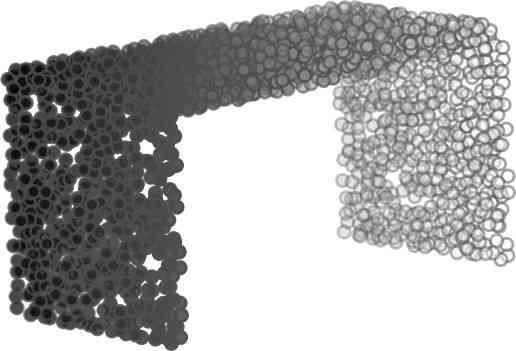}table $\to$ desk\\
%\hline
\includegraphics[height=0.070000\textwidth]{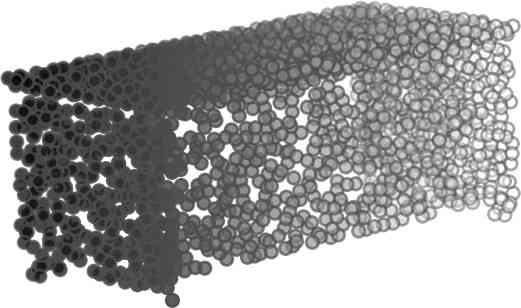}table $\to$ desk&
\includegraphics[height=0.070000\textwidth]{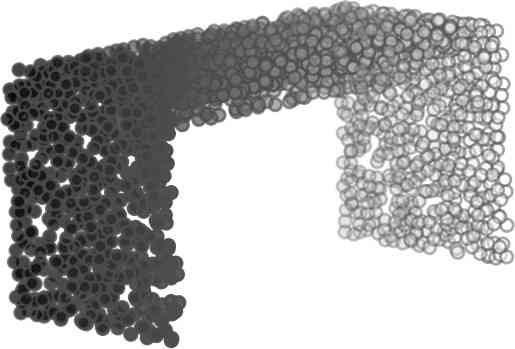}table $\to$ desk&
\includegraphics[height=0.070000\textwidth]{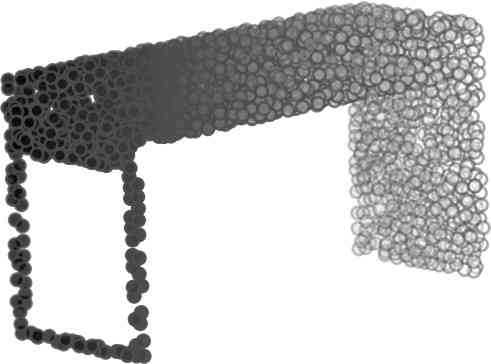}table $\to$ desk\\
%\hline
\includegraphics[height=0.070000\textwidth]{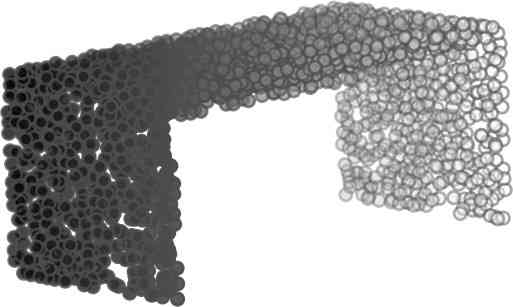}table $\to$ desk&
\includegraphics[height=0.070000\textwidth]{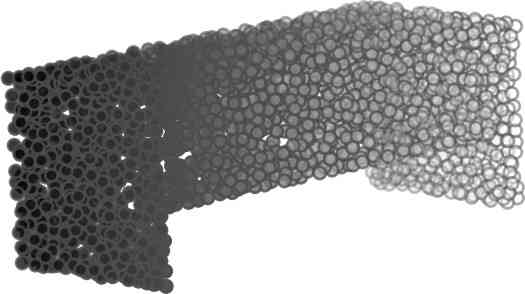}table $\to$ desk&
\includegraphics[height=0.070000\textwidth]{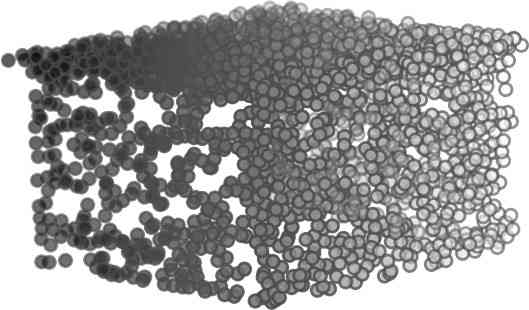}table $\to$ desk\\
%\hline
\includegraphics[height=0.070000\textwidth]{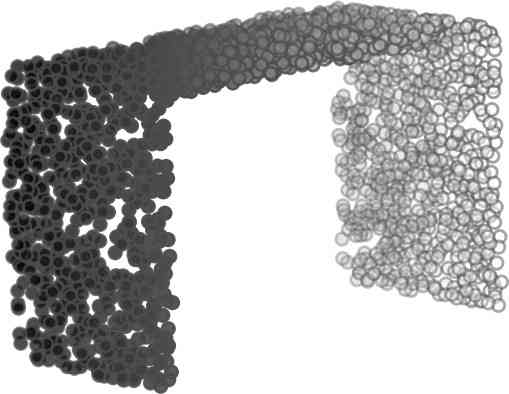}table $\to$ desk&
\includegraphics[height=0.070000\textwidth]{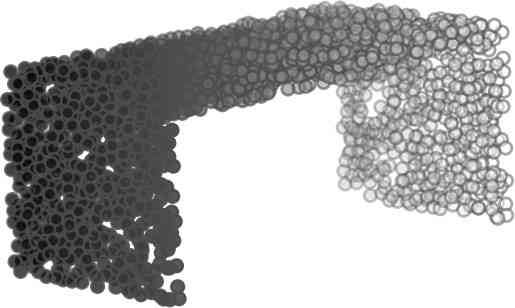}table $\to$ desk&
\includegraphics[height=0.070000\textwidth]{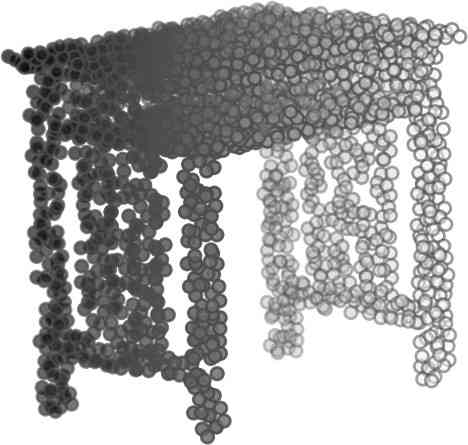}table $\to$ night\_stand\\
%\hline
\includegraphics[height=0.070000\textwidth]{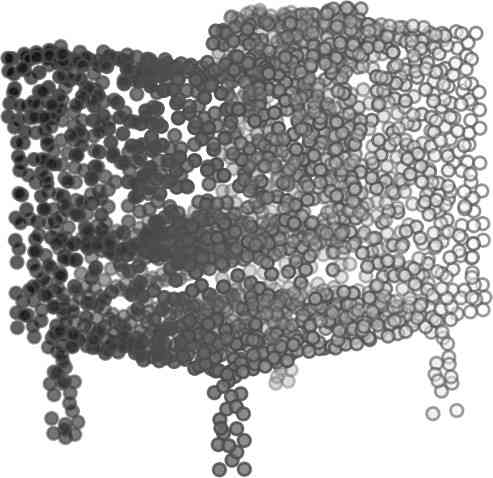}sofa $\to$ night\_stand&
\includegraphics[height=0.070000\textwidth]{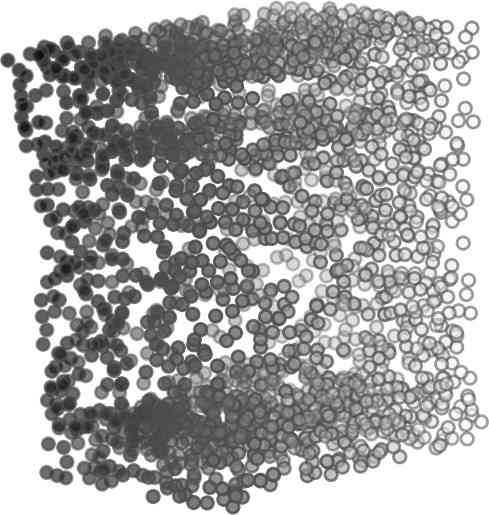}night\_stand $\to$ dresser&
\includegraphics[height=0.070000\textwidth]{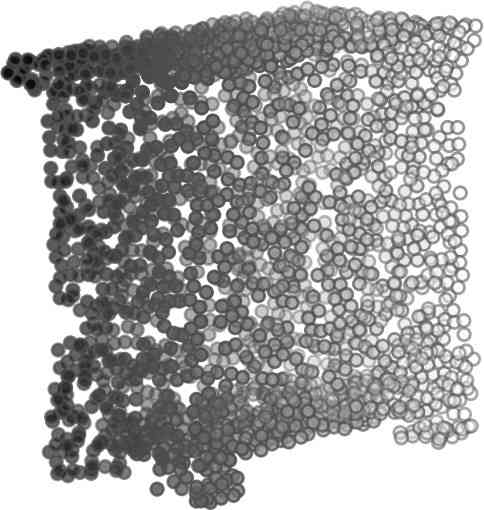}night\_stand $\to$ dresser\\
%\hline
\includegraphics[height=0.070000\textwidth]{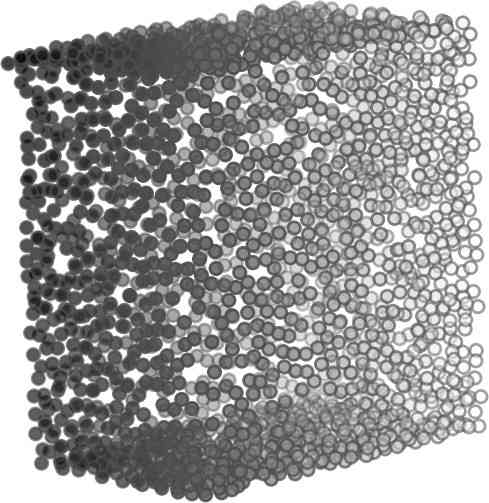}night\_stand $\to$ dresser&
\includegraphics[height=0.070000\textwidth]{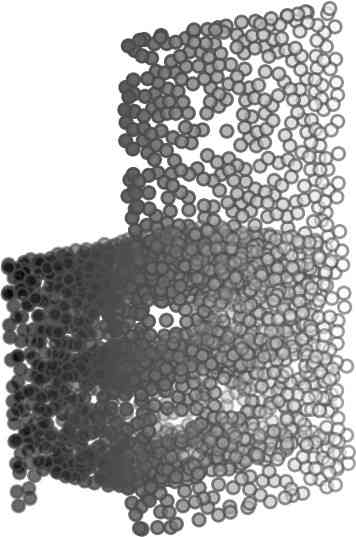}night\_stand $\to$ dresser&
\includegraphics[height=0.070000\textwidth]{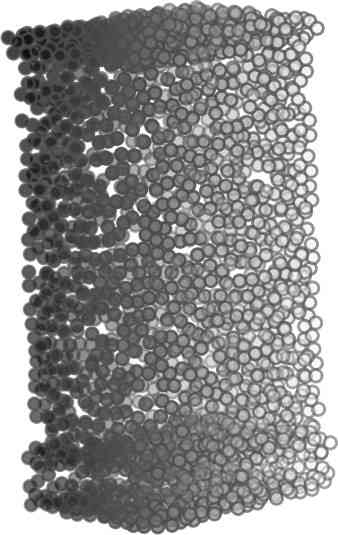}night\_stand $\to$ dresser\\
%\hline
\includegraphics[height=0.070000\textwidth]{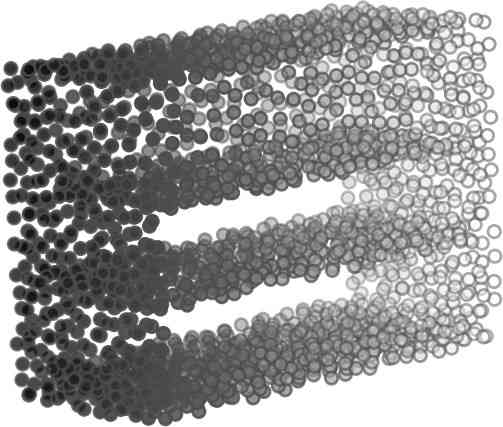}night\_stand $\to$ dresser&
\includegraphics[height=0.070000\textwidth]{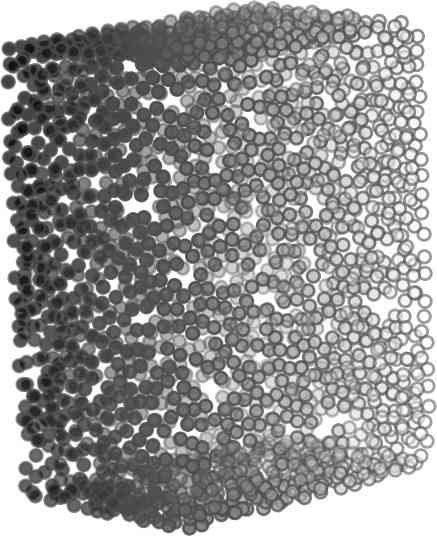}night\_stand $\to$ dresser&
\includegraphics[height=0.070000\textwidth]{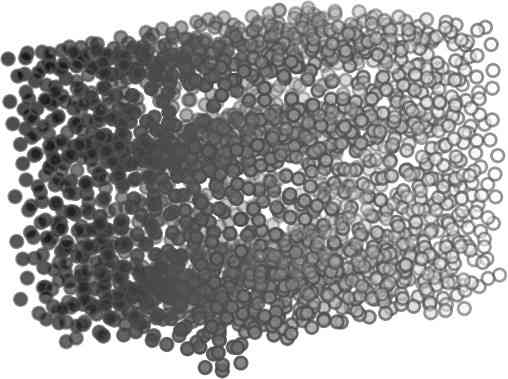}night\_stand $\to$ dresser\\
%\hline
\includegraphics[height=0.070000\textwidth]{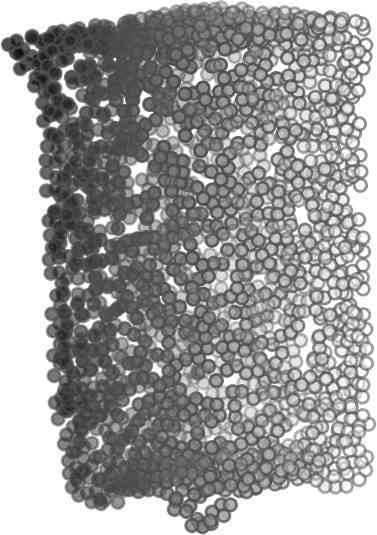}night\_stand $\to$ dresser&
\includegraphics[height=0.070000\textwidth]{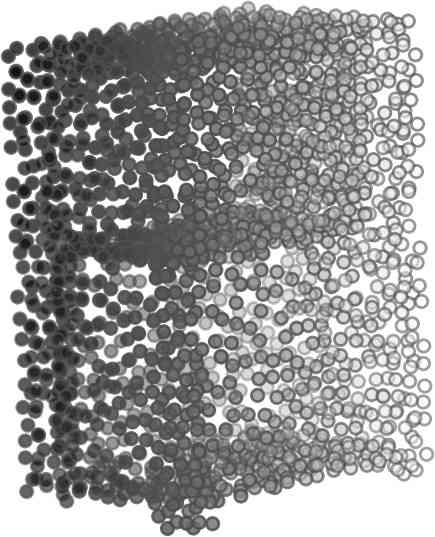}night\_stand $\to$ dresser&
\includegraphics[height=0.070000\textwidth]{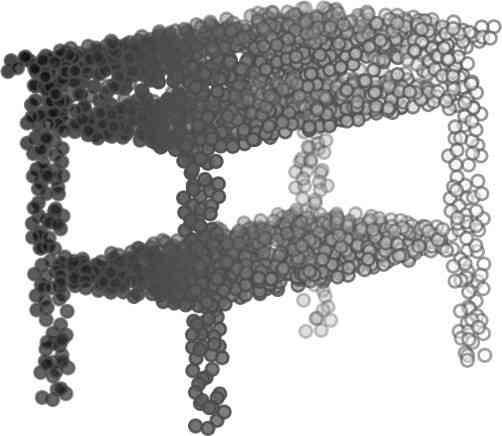}night\_stand $\to$ table\\
%\hline
\includegraphics[height=0.070000\textwidth]{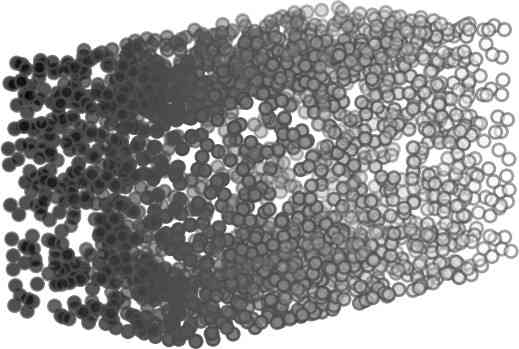}night\_stand $\to$ dresser&
\includegraphics[height=0.070000\textwidth]{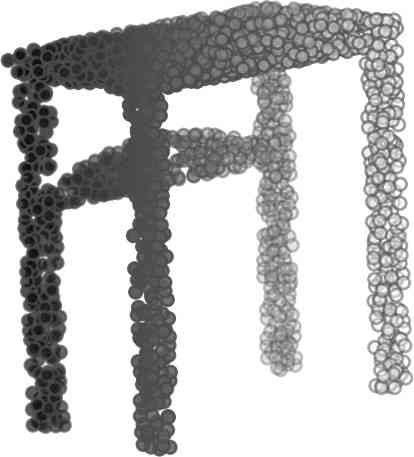}night\_stand $\to$ table&
\includegraphics[height=0.070000\textwidth]{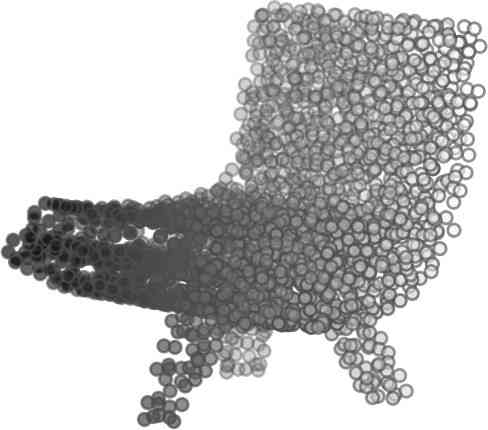}chair $\to$ bed\\
%\hline
\hline
\end{tabularx}
\end{table*}

\end{document}